%% file: main.tex
\documentclass{article}
\usepackage{multirow}

     \usepackage[preprint,nonatbib]{neurips_2020}

\usepackage[utf8]{inputenc} 
\usepackage[T1]{fontenc}    
\usepackage{hyperref}       
\usepackage{url}            
\usepackage{booktabs}       
\usepackage{amsfonts}       
\usepackage{nicefrac}       
\usepackage{microtype}      
\usepackage{float}
\usepackage{bbm}
\usepackage{amsthm,amsmath,amssymb}
\usepackage{mathtools}
\usepackage{subcaption}
\usepackage{threeparttable}

\usepackage{caption}
\usepackage{tikz}
\usetikzlibrary{arrows,automata, calc, positioning}
\usepackage{adjustbox}
\usepackage{wrapfig}

\usepackage{accents}

\newtheorem{proposition}{Proposition}[section]
\newtheorem{definition}{Definition}[section]

\newtheorem{lemma}{Lemma}[section]
\newtheorem{theorem}{Theorem}[section]
\newtheorem{algorithm}{Algorithm}[section]
\newtheorem{corollary}{Corollary}[section]

\newtheorem{assumption}{Assumption}[section]
\newtheorem{example}{Example}[section]

\title{A Simple and General Debiased Machine Learning Theorem with Finite Sample Guarantees}

\author{%
  Victor Chernozhukov \\
  MIT Economics \\
  \texttt{vchern@mit.edu} \\
   \And
   Whitney K. Newey \\
MIT Economics \\
   \texttt{wnewey@mit.edu} \\
   \And
   Rahul Singh \\
   MIT Economics \\
   \texttt{rahul.singh@mit.edu} \\
}

\begin{document}

\maketitle

\input{0_abstract}



\input{1_intro}
\input{2_related}
\input{3_problem}
\input{4_algorithm}
\input{5_dml}



\begin{ack}
The National Science Foundation provided partial financial support via
grants 1559172 and 1757140. Rahul Singh thanks the Jerry Hausman Dissertation Fellowship.
\end{ack}


\input{main.bbl}
\newpage


\appendix

\input{A_sim}
\input{B_local}
\input{C_discussion}
\input{D_proofs}
\input{E_proofs_local}

\end{document}

%% file: 0_abstract.tex
\begin{abstract}
   Debiased machine learning is a meta algorithm based on bias correction and sample splitting to calculate confidence intervals for functionals, i.e. scalar summaries, of machine learning algorithms. For example, an analyst may desire the confidence interval for a treatment effect estimated with a neural network. We provide a nonasymptotic debiased machine learning theorem that encompasses any global or local functional of any machine learning algorithm that satisfies a few simple, interpretable conditions. Formally, we prove consistency, Gaussian approximation, and semiparametric efficiency by finite sample arguments. The rate of convergence is $n^{-1/2}$ for global functionals, and it degrades gracefully for local functionals. Our results culminate in a simple set of conditions that an analyst can use to translate modern learning theory rates into traditional statistical inference. The conditions reveal a general double robustness property for ill posed inverse problems.
\end{abstract}

%% file: 1_intro.tex
\section{Introduction}\label{sec:intro}

The goal of this paper is to provide a useful technical result for analysts who desire confidence intervals for functionals, i.e. scalar summaries, of machine learning algorithms. For example, the functional of interest could be the average treatment effect of a medical intervention, and the machine learning algorithm could be a neural network trained on medical scans. Alternatively, the functional of interest could be the price elasticity of consumer demand, and the machine learning algorithm could be a kernel ridge regression trained on economic transactions. Treatment effects and price elasticities for a specific demographic are examples of localized functionals. In these various applications, confidence intervals are essential.

We provide a simple set of conditions that can be verified using the kind of rates provided by statistical learning theory. Unlike previous work, we provide a finite sample analysis for any global or local functional of any machine learning algorithm, without bootstrapping, subject to these simple and interpretable conditions. The machine learning algorithm may be estimating a nonparametric regression, a nonparametric instrumental variable regression, or some other nonparametric quantity. We provide conceptual and statistical contributions for the rapidly growing literature on debiased machine learning.

Conceptually, our result unifies, refines, and extends existing debiased machine learning theory for a broad audience. We unify finite sample results that are specific to particular functionals or machine learning algorithms. General asymptotic theory with abstract conditions already exists, which we refine to finite sample theory with simple conditions. In doing so, we uncover a new notion of double robustness for exactly identified ill posed inverse problems. A virtue of finite sample analysis is that it handles the case where the functional involves localization. We show how learning theory delivers inference.

Statistically, we provide results for the class of global functionals that are mean square continuous, and their local counterparts, using algorithms that have sufficiently fast finite sample learning rates. Formally, we prove (i) consistency, Gaussian approximation, and semiparametric efficiency for global functionals; and (ii) consistency and Gaussian approximation for local functionals. The analysis explicitly accounts for each source of error in any finite sample size. The rate of convergence is the parametric rate of $n^{-1/2}$ for global functionals, and it degrades gracefully to nonparametric rates for local functionals.

%% file: 2_related.tex
\section{Related work}\label{sec:related_main}

By focusing on functionals of nonparametric quantities, this paper continues the tradition of classic semiparametric statistics \cite{hasminskii1979nonparametric,robinson1988root,bickel1993efficient,newey1994asymptotic,andrews1994asymptotics,robins1995semiparametric,ai2003efficient}. Whereas classic semiparametric theory studies functionals of densities or regressions over low dimensional domains, we study functionals of machine learning algorithms over arbitrary domains. 
In classic semiparametric theory, an object called the Riesz representer appears in efficient influence functions and asymptotic variance calculations \cite{newey1994asymptotic}. For the same reasons, it appears in debiased machine learning confidence intervals.

In asymptotic inference, the Riesz representer is inevitable. A growing literature directly incorporates the Riesz representer into estimation, which amounts to debiasing known estimators. Doubly robust estimating equations serve this purpose \cite{robins1995semiparametric}. A geometric perspective emphasizes Neyman orthogonality: by debiasing, the learning problem for the functional becomes orthogonal to the learning problem for the nonparametric object \cite{chernozhukov2016locally,chernozhukov2018original,foster2019orthogonal}. An analytic perspective emphasizes the mixed bias property: by debiasing, the functional has bias equal to the product of certain learning rates \cite{chernozhukov2018original,rotnitzky2021characterization}. In this work, we focus on debiased machine learning with doubly robust estimating equations.

With debiasing alone, a key challenge remains: for inference, the function class in which the nonparametric quantity is learned must be Donsker \cite{van2006targeted,luedtke2016statistical,van2018targeted,qiu2021universal}, or it must have slowly increasing entropy \cite{belloni2013inference,belloni2014uniform,zhang2014confidence,javanmard2014confidence,vandegeer2014asymptotically}. However, popular nonparametric settings in machine learning may not satisfy this property. A solution to this challenging issue is to combine debiasing with sample splitting \cite{klaassen1987consistent}. The targeted \cite{zheng2011cross}
and debiased \cite{belloni2012sparse,chernozhukov2016locally,chernozhukov2018original} machine learning literatures provide this insight. In particular, debiased machine learning delivers sufficient conditions for asymptotic inference on functionals in terms of learning rates of the underlying nonparametric quantity and the Riesz representer. We complement prior results with a finite sample analysis. 

This paper subsumes \cite[Section 4]{singh2021debiased}.

%% file: 3_problem.tex
\section{Framework and examples}\label{sec:framework}

The general inference problem is to find a confidence interval for some scalar $\theta_0\text{ in }\mathbb{R}$ where
$
\theta_0=E\{m(W,\gamma_0)\}$, $\gamma_0$ is in $\Gamma$,
and $m:\mathcal{W}\times \mathbb{L}_2\rightarrow{\mathbb{R}}$ is an abstract formula. $W\text{ in }\mathcal{W}$ is a concatenation of random variables in the model excluding the outcome $Y\text{ in }\mathcal{Y}\subset\mathbb{R}$.  $\mathbb{L}_2$ is the space of functions of the form $\gamma:\mathcal{W}\rightarrow \mathbb{R}$ that are square integrable with respect to measure $\text{pr}$. $\Gamma$ is a linear subset of $\mathbb{L}_2$ known by the analyst, which may be $\mathbb{L}_2$ itself. 

Note that $\gamma_0$ may be the conditional expectation function $\gamma_0(w)=E(Y \mid W=w)$ or some other nonparametric quantity. For example, it could be the function defined as the solution to the ill posed inverse problem $E(Y \mid W_2=w_2)=E\{\gamma(W_1) \mid W_2=w_2\}$ where $W_1,W_2\subset W$. Such a function is called a nonparametric instrumental variable regression in econometrics \cite{newey2003instrumental}. We study the exactly identified case, which amounts to assuming completeness when $\Gamma=\mathbb{L}_2$ \cite{chen2018overidentification}. If $W_1=W_2$ then nonparametric instrumental variable regression simplifies into nonparametric regression. 

A local functional $\theta_0^{\lim}$ in $\mathbb{R}$ is a scalar that takes the form
$$
\theta^{\lim}_{0}=\lim_{h\rightarrow 0} \theta_0^h,\quad \theta_0^h=E\{m_h(W,\gamma_0)\}=E\{\ell_h(W_j) m(W,\gamma_0)\},\quad \gamma_0 \text{ in } \Gamma,
$$
where $\ell_h$ is a Nadaraya Watson weighting with bandwidth $h$ and $W_j$ is a scalar component of $W$. $\theta^{\lim}_{0}$ is a nonparametric quantity. However, it can be approximated by the sequence $(\theta_0^h)$. Each $\theta_0^h$ can be analyzed like $\theta_0$ above as long as we keep track of how certain quantities depend on $h$. By this logic, finite sample semiparametric theory for $\theta^h_0$ translates to finite sample nonparametric theory for $\theta_0^{\lim}$ up to some approximation error. In this sense, our analysis encompasses both global and local functionals.

To illustrate, we consider some classic functionals.

\begin{example}[Heterogeneous treatment effect estimated by neural network]\label{ex:CATE}
Let $Y$ be a health outcome. Let $W=(D,V,X)$ concatenate binary treatment $D$, covariate of interest $V$ such as age, and other covariates $X$ such as medical scans. Let $\gamma_0(d,v,x)=E(Y\mid D=d,V=v,X=x)$ be a function estimated by a neural network. Under the assumption of selection on observables, the heterogeneous treatment effect is 
$$
\textsc{CATE}(v)=E\{\gamma_0(1,V,X)-\gamma_0(0,V,X)\mid V=v\}=\lim_{h\rightarrow 0}E[\ell_h(V)\{\gamma_0(1,V,X)-\gamma_0(0,V,X)\}],$$ where
$
\ell_h(V)=(h\omega)^{-1}K\left\{(V-v)/h\right\}$, $\omega=E [h^{-1} K\left\{(V-v)/h\right\}]
$, and $K$ is a bounded and symmetric kernel that integrates to one.
\end{example}
The heterogeneous treatment effect is defined with respect to some interpretable, low dimensional characteristic $V$ such as age, race, or gender \cite{abrevaya2015estimating}. The same functional without the localization $\ell_h$ is the classic average treatment effect. See \cite{bibaut2017data} and \cite{colangelo2020double} for other meaningful localizations of average treatment effect.

\begin{example}[Regression discontinuity design estimated by random forest]\label{ex:RDD}
Let $Y$ be an educational outcome. Let $W=(D,X)$ concatenate test score variable $D$ and covariates $X$. Let $\gamma_0(d,x)=E(Y\mid D=d,X=x)$ be a function estimated by a random forest. Suppose the cutoff for a scholarship is the test score $D=0$. The regression discontinuity design parameter is
$$
\textsc{RDD}=\lim_{d\downarrow 0}E\{\gamma_0(d,X)\}-\lim_{d\uparrow 0}E\{\gamma_0(d,X)\}=\lim_{h\rightarrow 0}E\{\ell^{+}_{h}(D)\gamma_0(D,X)-\ell^{-}_{h}(D)\gamma_0(D,X)\},
$$
where
$
\ell^+_h(D)=(h\omega^+)^{-1} K\left\{(2D-h)/(2h)\right\}$, $\omega^+=E \left[h^{-1} K\left\{(2D-h)/(2h)\right\}\right]
$, 
$\ell^-_h(D)=(h\omega^-)^{-1} K\left\{(-2D-h)/(2h)\right\}$, $\omega^-=E \left[h^{-1} K\left\{(-2D-h)/(2h)\right\}\right],$ and $K$ vanishes outside of the interval $(-1/2,1/2)$.
\end{example}
The expressions for fuzzy regression discontinuity, exact kink, and fuzzy kink designs are similar. 

\begin{example}[Demand elasticity estimated by kernel instrumental variable regression]\label{ex:elasticity}
Let $Y$ be log quantity demanded of some good. Let $W=(D,X,Z)$ concatenate log price $D$, covariates $X$, and cost shifter $Z$. Let $\gamma_0(d,x)$ be defined as the solution to $E(Y \mid X=x, Z=z)=E\{\gamma(D,X) \mid X=x, Z=z\}$ estimated by a kernel instrumental variable regression \cite{singh2019kernel}. The demand elasticity is 
$$
\textsc{ELASTICITY}=E\left\{\frac{\partial }{\partial d} \gamma_0(D,X) \right\}.
$$
\end{example}
In Supplement~2, we present the additional example of heterogeneous average derivative estimated by lasso, which is useful when an analyst has access to data on household spending behavior. 

For our simple and general theorem, we require that the formula $m$ is mean square continuous.
\begin{assumption}[Linearity and mean square continuity]\label{assumption:cont}
Assume that the functional $\gamma\mapsto \mathbb{E}\{m(W,\gamma)\}$ is linear, and that there exist $\bar{Q}<\infty $ and $q>0$ such that
$
E\{m(W,\gamma)^2\}\leq \bar{Q} [E\{\gamma(W)^2\}]^q$ for all $\gamma\text{ in } \Gamma.
$
\end{assumption}
This condition will be key in Section~\ref{sec:dml}, where we reduce the problem of inference for $\theta_0$ into the problem of learning $(\gamma_0,\alpha^{\min}_0)$, where $\alpha^{\min}_0$ is introduced below. It is a powerful condition satisfied by many functionals of interest, or at least satisfied by their approximating sequences. Though the local functional $\theta_0^{\lim}$ does not satisfy Assumption~\ref{assumption:cont}, each approximating $\theta_0^h$ does. In particular, for each $m_h$ there exists some $\bar{Q}_h$ that depends on $h$. We keep track of $\bar{Q}$ in our analysis and subsequently consider $\bar{Q}=\bar{Q}_h$. See Theorem~\ref{thm:local} below for conditions that characterize $\bar{Q}_h$ in local functionals, including Examples~\ref{ex:CATE} and~\ref{ex:RDD}.

The restriction that $\gamma_0$ is in $\Gamma\subset \mathbb{L}_2$, where $\Gamma$ is some linear function space, is called a restricted model in semiparametric statistical theory. In learning theory, mean square rates are adaptive to the smoothness of $\gamma_0$, encoded by $\gamma_0\text{ in }\Gamma$. We quote a general Riesz representation theorem for restricted models.

\begin{lemma}[Riesz representation \cite{chernozhukov2018global}]\label{prop:RR}
Suppose Assumption~\ref{assumption:cont} holds. Further suppose $\gamma_0\text{ is in } \Gamma$. Then there exists a Riesz representer $\alpha_0\text{ in } \mathbb{L}_2$ such that for all $\gamma$ in $\Gamma$, 
$
E\{m(W,\gamma)\}=E\{\alpha_0(W)\gamma(W)\}.
$
There exists a unique minimal Riesz representer $\alpha_0^{\min}\text{ in } closure(\Gamma)$ that satisfies this equation, obtained by projecting any $\alpha_0$ onto $\Gamma$. Moreover, denoting by $\bar{M}$ the operator norm of $\gamma\mapsto E\{m(W,\gamma)\}$, we have that
$
[E\{\alpha_0^{\min}(W)^2\}]^{1/2}=\bar{M} \leq \bar{Q}^{1/2}<\infty.
$
\end{lemma}
The condition $\bar{M}<\infty$ is enough for the conclusions of Lemma~\ref{prop:RR} to hold. Since $\bar{M}\leq \bar{Q}^{1/2}$, $\bar{Q}<\infty$ in Assumption~\ref{assumption:cont} is a sufficient condition. Nonetheless, we assume $\bar{Q}<\infty$ because mean square continuity plays a central role in the main results of Section~\ref{sec:dml}. In Examples~\ref{ex:CATE} and~\ref{ex:RDD}, with propensity score $\pi_0(v,x)$,
    $$
    \alpha_0(d,v,x)=\ell_h(v)\left\{\frac{d}{\pi_0(v,x)}-\frac{1-d}{1-\pi_0(v,x)}\right\};\quad
    \alpha^+_0(d,x)=\ell_h^{+}(d),
    \quad 
     \alpha^-_0(d,x)=\ell_h^{-}(d).
    $$
Riesz representation delivers a doubly robust formulation of the target $\theta_0\text{ in }\mathbb{R}$. For the case where $\gamma_0(w)$ is defined as a nonparametric regression in $\Gamma$ or projection onto $\Gamma$, consider the estimating equation
$$
\theta_0=E[m(W,\gamma_0)+\alpha^{\min}_0(W)\{Y-\gamma_0(W)\}].
$$
This formulation is doubly robust since it remains valid if either $\gamma_0$ or $\alpha^{\min}_0$ is correct: for all $(\gamma,\alpha)$ in $\Gamma$,
$$
\theta_0=E[m(W,\gamma_0)+\alpha(W)\{Y-\gamma_0(W)\}]=E[m(W,\gamma)+\alpha^{\min}_0(W)\{Y-\gamma(W)\}].
$$
The term $\alpha(w)\{y-\gamma(w)\}$ serves as a bias correction for the term $m(w,\gamma)$. We view $(\gamma_0,\alpha^{\min}_0)$ as nuisance parameters that we must learn in order to learn and infer $\theta_0$. While any Riesz representer $\alpha_0$ will suffice for valid learning and inference of $\theta_0=E\{m(W,\gamma_0)\}$ under correct specification of $\gamma_0$ as the regression $E(Y \mid W=w)$ in $\Gamma$, the minimal Riesz representer $\alpha_0^{\min}$ confers specification robust inference and semiparametric efficiency for estimating $\theta_0=E\{m(W,\gamma_0)\}$ when $\gamma_0$ is only the projection of $E(Y \mid W=w)$ onto $\Gamma$; see \cite[Theorem 4.2]{chernozhukov2018global}.

If $\gamma_0(w)$ is defined as the solution to an ill posed inverse problem, then the appropriate Riesz representer is defined as the solution to another ill posed inverse problem \cite{severini2012efficiency,ichimura2021influence}. The relevant nuisance parameters are $(\gamma_0,\alpha^{\min}_0)$ defined as unique solutions $(\gamma,\alpha)$ to
$$
E(Y \mid W_2=w_2)=E\{\gamma(W_1) \mid W_2=w_2\},\quad \eta_0^{\min}(w_1)=E\{\alpha(W_2) \mid W_1=w_1\},
$$
where $\eta_0^{\min}$ is the minimal Riesz representer satisfying $E\{m(W_1,\gamma)\}=E\{\eta_0(W_1)\gamma(W_1)\}$ for all $\gamma$ in $\Gamma$ from Lemma~\ref{prop:RR}. Uniqueness is due to the assumption of exact identification, which amounts to completeness when $\Gamma=\mathbb{L}_2$. In Example~\ref{ex:elasticity}, $w_1=(d,x)$, $w_2=(z,x)$, and $\eta_0(d,x)=-\partial_d  
\log f(d\mid x)$ where $f(d \mid x)$ is a conditional density. This abuse of notation allows us to state unified results. The estimating equation is 
$$
\theta_0=E[m(W_1,\gamma_0)+\alpha^{\min}_0(W_2)\{Y-\gamma_0(W_1)\}].
$$ 
A new insight of this work is that, for any mean square continuous functional, $n^{-1/2}$ Gaussian approximation is still possible if either $\gamma_0$ or $\alpha^{\min}_0$ is the solution to a mildly, rather than severely, ill posed inverse problem; the doubly robust formulation confers double robustness to ill posedness.

%% file: 4_algorithm.tex
\section{Algorithm}\label{sec:algorithm}

Our goal is general purpose learning and inference for the target parameter $\theta_0\text{ in }\mathbb{R}$ that is a mean square continuous functional of $\gamma_0\text{ in } \Gamma$. Lemma~\ref{prop:RR} demonstrates that any such $\theta_0$ has a unique minimal representer $\alpha_0^{\min}\text{ in } \Gamma$. In this section, we describe a meta algorithm to turn estimators $\hat{\gamma}$ of $\gamma_0$ and $\hat{\alpha}$ of $\alpha_0^{\min}$ into an estimator $\hat{\theta}$ of $\theta_0$ such that $\hat{\theta}$ has a valid and practical confidence interval. Recall that $\hat{\gamma}$ may be any machine learning algorithm. To preserve this generality, we do not instantiate a choice of $\hat{\gamma}$; we treat it as a black box. In subsequent analysis, we will only require that $\hat{\gamma}$ converges to $\gamma_0$ in mean square error. This mean square rate is guaranteed by existing statistical learning theory. 

The target estimator $\hat{\theta}$ as well as its confidence interval will depend on nuisance estimators $\hat{\gamma}$ and $\hat{\alpha}$. We refrain from instantiating the estimator $\hat{\alpha}$ for $\alpha_0^{\min}$. As we will see in subsequent analysis, the general theory only requires that $\hat{\alpha}$ converges to $\alpha^{\min}_0$ in mean square error. A recent literature provides $\hat{\alpha}$ estimators with fast rates inspired by the Dantzig selector \cite{chernozhukov2018global}, lasso \cite{chernozhukov2018learning,smucler2019unifying,avagyan2021high}, adversarial neural networks \cite{chernozhukov2020adversarial,kallus2021causal}, and kernel ridge regression \cite{singh2021debiased}.

\begin{algorithm}[Debiased machine learning]\label{alg:target}
Given a sample $(Y_i,W_i)$ $(i=1,...,n)$, partition the sample into folds $(I_{\ell})$ $(\ell=1,...,L)$. Denote by $I_{\ell}^c$ the complement of $I_{\ell}$.
\begin{enumerate}
    \item For each fold $\ell$, estimate $\hat{\gamma}_{\ell}$ and $\hat{\alpha}_{\ell}$ from observations in $I_{\ell}^c$.
    \item Estimate $\theta_0$ as
    $
  \hat{\theta}=n^{-1}\sum_{\ell=1}^L\sum_{i\in I_{\ell}} [m(W_i,\hat{\gamma}_{\ell})+\hat{\alpha}_{\ell}(W_i)\{Y_i-\hat{\gamma}_{\ell}(W_i)\}]
    $.
    \item Estimate its $(1-a) 100$\% confidence interval as
    $
    \hat{\theta}\pm c_{a}\hat{\sigma} n^{-1/2}$, where $c_{a}$ is the $1-a/2$ quantile of the standard Gaussian and $\hat{\sigma}^2=n^{-1}\sum_{\ell=1}^L\sum_{i\in I_{\ell}} [m(W_i,\hat{\gamma}_{\ell})+\hat{\alpha}_{\ell}(W_i)\{Y_i-\hat{\gamma}_{\ell}(W_i)\}-\hat{\theta}]^2
    $. 
\end{enumerate}
\end{algorithm}

This meta algorithm can be seen as an extension of classic one step corrections \cite{pfanzagl1982lecture} amenable to the use of modern machine learning, and it has been termed debiased machine learning \cite{chernozhukov2018original}. It departs from targeted machine learning inference with a finite sample \cite{van2017finite,cai2020nonparametric} in a few ways. On the one hand, it avoids iteration and bootstrapping, thereby simplifying computation. On the other hand, it does not involve substitution, which would ensure that the estimator obeys additional meaningful constraints. See \cite{chernozhukov2018learning} for an algorithm that combines the two approaches.

%% file: 5_dml.tex
\section{Validity of confidence interval}\label{sec:dml}

We write this section at a high level of generality so it can be used by analysts working on a variety of problems. We assume a few simple and interpretable conditions and consider black box estimators $(\hat{\gamma},\hat{\alpha})$. We prove by finite sample arguments that $\hat{\theta}$ defined by Algorithm~\ref{alg:target} is consistent, and that its confidence interval is valid and semiparametrically efficient. Towards this end, define the oracle moment function
$$
 \psi_0(w)=\psi(w,\theta_0,\gamma_0,\alpha^{\min}_0), \quad \psi(w,\theta,\gamma,\alpha)=m(w,\gamma)+\alpha(w)\{y-\gamma(w)\}-\theta.
$$
Its moments are $\sigma^2=E\{\psi_0(W)^2\}$, $\kappa^3=E\{|\psi_0(W)|^3\}$, and $\zeta^4=E\{\psi_0(W)^4\}$. Write the Berry Esseen constant as $c^{BE}=0.4748$ \cite{shevtsova2011absolute}. The result will be in terms of abstract mean square rates.

\begin{definition}[Mean square error]
Write the mean square error $\mathcal{R}(\hat{\gamma}_{\ell})$ and the projected mean square error $\mathcal{P}(\hat{\gamma}_{\ell})$ of $\hat{\gamma}_{\ell}$ trained on observations indexed by $I^c_{\ell}$ as
$$
    \mathcal{R}(\hat{\gamma}_{\ell})=E[\{\hat{\gamma}_{\ell}(W)-\gamma_0(W)\}^2\mid I^c_{\ell}],\quad 
     \mathcal{P}(\hat{\gamma}_{\ell})=E([ E\{\hat{\gamma}_{\ell}(W_1)-\gamma_0(W_1)\mid W_2, I^c_{\ell}\} ]^2\mid I^c_{\ell}).
$$
Likewise define $\mathcal{R}(\hat{\alpha}_{\ell})$ and $\mathcal{P}(\hat{\alpha}_{\ell})$.
\end{definition}
Statistical learning theory provides rates of this form, where $I^c_{\ell}$ is a training set and $W$ is a test point. In the case of nonparametric regression, $\mathcal{R}(\hat{\gamma}_{\ell})$ or $\mathcal{R}(\hat{\alpha}_{\ell})$ typically has a fast rate between $n^{-1/2}$ and $n^{-1}$. In the case of nonparametric instrumental variable regression, $\mathcal{R}(\hat{\gamma}_{\ell})$ and $\mathcal{R}(\hat{\alpha}_{\ell})$ typically have rates slower than $n^{-1/2}$ due to ill posedness, but $\mathcal{P}(\hat{\gamma}_{\ell})$ or $\mathcal{P}(\hat{\alpha}_{\ell})$ may have a fast rate \cite{blundell2007semi,singh2019kernel,dikkala2020minimax}. Our main result is a finite sample Gaussian approximation.

\begin{theorem}[Finite sample Gaussian approximation]\label{thm:dml}Suppose Assumption~\ref{assumption:cont} holds,
$
E[\{Y-\gamma_0(W)\}^2 \mid W]\leq \bar{\sigma}^2,$
and
$\|\alpha^{\min}_0\|_{\infty}\leq\bar{\alpha}.
$
Then with probability $1-\epsilon$,
$$
\sup_{z\in\mathbb{R}} \left| \text{\normalfont pr} \left\{\frac{n^{1/2}}{\sigma}(\hat{\theta}-\theta_0)\leq z\right\}-\Phi(z)\right|\leq c^{BE}\left(\frac{\kappa}{\sigma}\right)^3 n^{-1/2}+\frac{\Delta}{(2\pi)^{1/2}}+\epsilon,
$$
where $\Phi(z)$ is the standard Gaussian cumulative distribution function and
$$
\Delta=\frac{3 L}{\epsilon   \sigma}\left[(\bar{Q}^{1/2}+\bar{\alpha})\{\mathcal{R}(\hat{\gamma}_{\ell})\}^{q/2}+\bar{\sigma}\{\mathcal{R}(\hat{\alpha}_{\ell})\}^{1/2}+\{n \mathcal{R}(\hat{\gamma}_{\ell}) \mathcal{R}(\hat{\alpha}_{\ell}) \}^{1/2}\right].
$$
If in addition $\|\hat{\alpha}_{\ell}\|_{\infty}\leq\bar{\alpha}'$ then the same result holds updating $\Delta$ to be
$$
\frac{4 L}{\epsilon^{1/2}  \sigma}\left[(\bar{Q}^{1/2}+\bar{\alpha}+\bar{\alpha}')\{\mathcal{R}(\hat{\gamma}_{\ell})\}^{q/2}
    +\bar{\sigma}\{\mathcal{R}(\hat{\alpha}_{\ell})\}^{1/2}\right]+\frac{1}{\sigma}[\{n\mathcal{P}(\hat{\gamma}_{\ell})\mathcal{R}(\hat{\alpha}_{\ell})\}^{1/2} \wedge \{n\mathcal{R}(\hat{\gamma}_{\ell})\mathcal{P}(\hat{\alpha}_{\ell})\}^{1/2}].
$$
For local functionals, further suppose approximation error of size $\Delta_h=
n^{1/2} \sigma_h^{-1}|\theta_0^h-\theta_0^{\lim}|$. Then the same result holds replacing $(\hat{\theta},\theta_0,\Delta)$ with $(\hat{\theta}^h,\theta_0^{\lim},\Delta+\Delta_h)$.
\end{theorem}

Theorem~\ref{thm:dml} is a finite sample Gaussian approximation for debiased machine learning with black box $(\hat{\gamma}_{\ell},\hat{\alpha}_{\ell})$. It degrades gracefully if the parameters $(\bar{Q},\bar{\sigma},\bar{\alpha},\bar{\alpha}')$ diverge relative to $n$ and the learning rates. Note that $\bar{\alpha}'$ is a bound on the chosen estimator $\hat{\alpha}_{\ell}$ that can be imposed by censoring extreme evaluations. Theorem~\ref{thm:dml} is a finite sample refinement of the asymptotic black box result in \cite{chernozhukov2016locally}. 

In the bound $\Delta$, the expression $\{n \mathcal{R}(\hat{\gamma}_{\ell}) \mathcal{R}(\hat{\alpha}_{\ell})\}^{1/2}$ allows a tradeoff: one of the learning rates may be slow, as long as the other is sufficiently fast to compensate. It is easily handled in the case of nonparametric regression, where $\mathcal{R}(\hat{\gamma}_{\ell})$ or $\mathcal{R}(\hat{\alpha}_{\ell})$ typically has a fast rate. However, the expression may diverge in the case of nonparametric instrumental variable regression, where both rates may be slow due to ill posedness. 

The refined bound provides an alternative path to Gaussian approximation, replacing $\{n \mathcal{R}(\hat{\gamma}_{\ell}) \mathcal{R}(\hat{\alpha}_{\ell}) \}^{1/2}$ with the minimum of $\{n\mathcal{P}(\hat{\gamma}_{\ell})\mathcal{R}(\hat{\alpha}_{\ell})\}^{1/2}$ and $\{n\mathcal{R}(\hat{\gamma}_{\ell})\mathcal{P}(\hat{\alpha}_{\ell})\}^{1/2}$. Importantly, the projected mean square error $\mathcal{P}(\hat{\gamma}_{\ell})$ can have a fast rate even when the mean square error $\mathcal{R}(\hat{\gamma}_{\ell})$ has a slow rate because its definition sidesteps ill posedness. Moreover, the analyst only needs $\mathcal{P}(\hat{\gamma}_{\ell})$ fast enough to compensate for the ill posedness encoded in $\mathcal{R}(\hat{\alpha}_{\ell})$, or  $\mathcal{P}(\hat{\alpha}_{\ell})$ fast enough to compensate for the ill posedness encoded in $\mathcal{R}(\hat{\gamma}_{\ell})$. This general and finite sample characterization of double robustness to ill posedness appears to be new. In independent work, \cite{kallus2021causal} document an asymptotic special case of this result for a specific global functional and specific nuisance estimators; see Supplement~3.

By Theorem~\ref{thm:dml}, the neighborhood of Gaussian approximation scales as $\sigma n^{-1/2}$. If $\sigma$ is a constant, then the rate of convergence is $n^{-1/2}$, i.e. the parametric rate. If $\sigma$ is a diverging sequence, then the rate of convergence degrades gracefully to nonparametric rates. A precise characterization of $\sigma$ is possible, which we provide in Supplement~2 and summarize here.
It turns out that global functionals have $\sigma$ that is constant, while local functionals have $\sigma=\sigma_h$ that is a diverging sequence. We emphasize which quantities are diverging sequences for local functionals by indexing with the bandwidth $h$.

\begin{theorem}[Characterization of key quantities]\label{thm:local}
If noise has finite variance then $\bar{\sigma}^2<\infty$. Suppose bounded moment and heteroscedasticity conditions defined in Supplement~2 hold. Then for global functionals
$
\kappa/\sigma \lesssim \sigma \asymp \bar{M} < \infty$; $\kappa, \zeta\lesssim \bar{M}^2\leq \bar{Q}<\infty$; and $\bar{\alpha}<\infty.$
Suppose bounded moment, heteroscedasticity, density, and derivative conditions defined in Supplement~2 hold. Then for local functionals
$
\kappa_h/\sigma_h \lesssim h^{-1/6}$, $\sigma_h \asymp \bar{M}_h \asymp h^{-1/2} $, $\kappa_h\lesssim h^{-2/3}$, $\zeta_h\lesssim h^{-3/4}$, 
$
\bar{Q}_h\lesssim h^{-2}$, $\bar{\alpha}_h\lesssim h^{-1}$, and $\Delta_h \lesssim n^{1/2} h^{\mathsf{v}+1/2}
$
where $\mathsf{v}$ is the order of differentiability defined in Supplement~2.
\end{theorem}
For global functionals, $(\bar{Q},\bar{\alpha})$ are finite constants that depend on the problem at hand. For example, for treatment effects a sufficient condition is that the propensity score is bounded away from zero and one. For derivatives, a sufficient condition is that $\Gamma$ satisfies Sobolev conditions. For local functionals, we handle $(\bar{Q}_h,\bar{\alpha}_h)$ on a case by case basis. See Supplement~2 for interpretable and complete characterizations.

Observe that the finite sample Gaussian approximation in Theorem~\ref{thm:dml} is in terms of the true asymptotic variance $\sigma^2$. We now provide a guarantee for its estimator $\hat{\sigma}^2$.

\begin{theorem}[Variance estimation]\label{thm:var}
Suppose Assumption~\ref{assumption:cont} holds, $
E[\{Y-\gamma_0(W)\}^2 \mid W]\leq \bar{\sigma}^2,$ and $\|\hat{\alpha}_{\ell}\|_{\infty}\leq\bar{\alpha}'.
$
Then with probability $1-\epsilon'$,
$
|\hat{\sigma}^2-\sigma^2|\leq \Delta'+2(\Delta')^{1/2}\{(\Delta'')^{1/2}+\sigma\}+\Delta'',
$ where
$$
 \Delta'=4(\hat{\theta}-\theta_0)^2+\frac{24 L}{\epsilon'}\left[\{\bar{Q}+(\bar{\alpha}')^2\}\mathcal{R}(\hat{\gamma}_{\ell})^q+\bar{\sigma}^2\mathcal{R}(\hat{\alpha}_{\ell})\right],\quad 
    \Delta''=\left(\frac{2}{\epsilon'}\right)^{1/2}\zeta^2 n^{-1/2}.
$$
\end{theorem}
Theorem~\ref{thm:var} is a finite sample variance estimation guarantee. It degrades gracefully if the parameters $(\bar{Q},\bar{\sigma},\bar{\alpha}')$ diverge relative to $n$ and the learning rates. Theorems~\ref{thm:dml} and~\ref{thm:var} immediately imply simple, interpretable conditions for validity of the confidence interval. We conclude by summarizing these conditions.

\begin{corollary}[Confidence interval]\label{cor:CI}
Suppose Assumption~\ref{assumption:cont} holds as well as the following regularity and learning rate conditions, as $n\rightarrow \infty$ and as $h\rightarrow 0$:
$$
E[\{Y-\gamma_0(W)\}^2 \mid W]\leq \bar{\sigma}^2,\quad \|\alpha^{\min}_0\|_{\infty}\leq\bar{\alpha},\quad  \|\hat{\alpha}_{\ell}\|_{\infty}\leq\bar{\alpha}', \quad  \left\{\left(\kappa/\sigma\right)^3+\zeta^2\right\}n^{-1/2}\rightarrow0;
$$  
\begin{enumerate}
    \item $\left(\bar{Q}^{1/2}+\bar{\alpha}/\sigma+\bar{\alpha}'\right)\{\mathcal{R}(\hat{\gamma}_{\ell})\}^{q/2}=o_p(1)$;
    \item $\bar{\sigma}\{\mathcal{R}(\hat{\alpha}_{\ell})\}^{1/2}=o_p(1)$;
    \item $[\{n \mathcal{R}(\hat{\gamma}_{\ell}) \mathcal{R}(\hat{\alpha}_{\ell})\}^{1/2} \wedge \{n\mathcal{P}(\hat{\gamma}_{\ell})\mathcal{R}(\hat{\alpha}_{\ell})\}^{1/2} \wedge \{n\mathcal{R}(\hat{\gamma}_{\ell})\mathcal{P}(\hat{\alpha}_{\ell})\}^{1/2}]/\sigma =o_p(1)$.
\end{enumerate}
Then the estimator $\hat{\theta}$ in Algorithm~\ref{alg:target} is consistent and asymptotically Gaussian, and the confidence interval in Algorithm~\ref{alg:target} includes $\theta_0$ with probability approaching the nominal level. Formally,
$$
\hat{\theta}=\theta_0+o_p(1),\quad \sigma^{-1}n^{1/2}(\hat{\theta}-\theta_0)\leadsto\mathcal{N}(0,1),\quad \text{\normalfont pr} \left\{\theta_0 \text{ in }  \left(\hat{\theta}\pm c_a\hat{\sigma} n^{-1/2} \right)\right\}\rightarrow 1-a.
$$
For local functionals, if $\Delta_h \rightarrow 0$, then the same result holds replacing $(\hat{\theta},\theta_0)$ with $(\hat{\theta}^h,\theta_0^{\lim})$.
\end{corollary}

%% file: A_sim.tex
\section{Simulations}

We present simulations for Example~\ref{ex:CATE}: heterogeneous treatment effect estimated by neural network. In addition, we present results for heterogeneous treatment effect estimated by random forest and lasso.

Recall that the localized functional is 
$$
\textsc{CATE}(v)=\lim_{h\rightarrow 0}E[\ell_{h,v}(V)\{\gamma_0(1,V,X)-\gamma_0(0,V,X)\}],$$ where
$
\ell_{h,v}(V)=(h\omega)^{-1}K\left\{(V-v)/h\right\}$ and $\omega=E [h^{-1} K\left\{(V-v)/h\right\}].
$
$V$ is an interpretable, low dimensional characteristic such as age, race, or gender. We implement the heterogeneous treatment effect design of \cite{abrevaya2015estimating}, 
where $\textsc{CATE}(v)=v(1+2v)^2(v-1)^2$ and $v$ is a value in the interval $(-0.5,0.5)$. A single observations consists of the tuple $(Y_i,D_i,V_i,X_i)$ for outcome, treatment, covariate of interest, and other covariates. In this design, $Y_i,D_i,V_i$ are in $\mathbb{R}$ and $X_i$ is in $\mathbb{R}^3$.

A single observation is generated as follows. Draw the latent variables $(\epsilon_{ij})$ $(j=1,...,4)$ independently and identically from the uniform distribution $ \mathcal{U}(-0.5,0.5)$. Then set the covariates $(V_i,X_i)$ according to
$$
V_i=\epsilon_{i1},\quad X_{i1}=1+2V_i+\epsilon_{i2},\quad X_{i2}=1+2V_i+\epsilon_{i3},\quad X_{i3}=(V_i-1)^2+\epsilon_{i4}.
$$
Under the assumption of selection on observables, treatment assignment is as good as random conditional on $(V_i,X_i)$. Draw the treatment $D_i$ from the Bernoulli distribution with parameter $\Lambda \{1/2(V_i+X_{i1}+X_{i2}+X_{i3})\}$ where $\Lambda$ is the logistic link function. Finally, calculate outcome $Y_i$ as $0$ if $D_i=0$ and $V_i X_{i1} X_{i2} X_{i3}+\nu_i$ if $D_i=1$, where the response noise $\nu_i$ is independently drawn from the Gaussian distribution $\mathcal{N}(0,1/16)$. A random sample consists of $n=100$ such observations $(Y_i,D_i,V_i,X_i)$ $(i=1,...,n)$.

We implement different variations of Algorithm~\ref{alg:target} with $L=5$ folds. Across variations, we use a lasso estimator $\hat{\alpha}$ for the minimal Riesz representer $\alpha_0^{\min}$ \cite{chernozhukov2018learning}. We consider different  estimators $\hat{\gamma}$ for the nonparametric regression $\gamma_0$: neural network, random forest, and lasso. We consider both low dimensional and high dimensional variations. In the low dimensional variation, the estimators $(\hat{\alpha}_{\ell},\hat{\gamma}_{\ell})$ use $(D_i,V_i,X_i)$ $(i \text{ \normalfont in } I^c_{\ell})$ as well as their interactions. In the high dimensional variation, the estimators $(\hat{\alpha}_{\ell},\hat{\gamma}_{\ell})$ use fourth order polynomials of $(D_i,V_i,X_i)$ $(i \text{ \normalfont in } I^c_{\ell})$.

Some tuning choices are necessary. We follow the default hyperparameter settings to tune the lasso Riesz representer and lasso regression from \cite{chernozhukov2018learning}. We implement the neural network with a single hidden layer of eight neurons and the random forest with 1000 trees as in \cite{chernozhukov2018original}.  Finally, to tune the bandwidth, we use the heuristic $h=c_h\hat{\sigma}_vn^{-0.2}$ \cite{colangelo2020double}, where $\hat{\sigma}^2_v$ is the sample variance of $(V_i)$ $(i=1,...,n)$. The bandwidth hyperparameter $c_h$ is chosen by the analyst. We evaluate robustness of coverage with respect to hyperparameter values $c_h=0.25, 0.50, 1.00$ below. Empirically, we find that $c_h=0.25$ and $c_h=0.50$ work well.

For each choice of nonparametric regression estimator $\hat{\gamma}$, whether neural network, random forest, or lasso, and for each choice of specification, whether low or high dimensional, we report a coverage table summarizing 500 simulations. The initial columns denote the grid value $v$, the corresponding heterogeneous treatment effect $\textsc{CATE}(v)$, and the bandwidth hyperparamter value $c_h$. The subsequent columns calculate the average point estimate and the average standard error across the 500 simulations for this choice of $\{v,\textsc{CATE}(v),c_h\}$. The final columns report what percentage of the 500 confidence intervals contain the true value $\textsc{CATE}(v)$ compared to the theoretical benchmarks of $80\%$ and $95\%$, respectively.

For the low dimensional regime, Tables~\ref{tab:nn_low},~\ref{tab:rf_low}, and~\ref{tab:lasso_low} summarize results for neural network, random forest, and lasso, respectively. With bandwidth hyperparameter values $c_h=0.25$ and $c_h=0.50$, coverage is close to the nominal level across $\hat{\gamma}$ estimators and across grid values $v=-0.25, 0.00, 0.25$. Neural network and random forest have comparable performance. Lasso has higher bias and compensates with higher variance for the grid value $v=0.25$.

\begin{table}[H]
\centering
  \begin{threeparttable}
    \caption{Low dimensional coverage simulation with neural network}
     \begin{tabular}{lcccccc}
       $v$& $\textsc{CATE}(v)$ & Tuning &  Ave. Est. & Ave. S.E. &  80\% Cov. & 95\% Cov. \\[5pt]
-0.25 & -0.10 & 0.25 & -0.10 & 0.05 &  83\% & 94\% \\
-0.25 & -0.10 & 0.50 & -0.10 & 0.04  & 85\% & 95\% \\
-0.25 & -0.10 & 1.00 & -0.08 & 0.03  & 71\% & 88\% \\
0.00 & 0.00 & 0.25 & 0.00 & 0.04 &  78\% & 95\% \\
0.00 & 0.00 & 0.50 & 0.00 & 0.03  & 78\% & 94\% \\
0.00 & 0.00 & 1.00 & 0.02 & 0.02  & 62\% & 85\% \\
0.25 & 0.32 & 0.25 & 0.31 & 0.12 &  85\% & 92\% \\
0.25 & 0.32 & 0.50 & 0.30 & 0.09  & 85\% & 93\% \\
0.25 & 0.32 & 1.00 & 0.28 & 0.06  & 76\% & 88\% 
     \end{tabular}
     \label{tab:nn_low}
    \begin{tablenotes}
      \small
      \item Ave., average; Est., estimate; S.E., standard error; Cov., coverage. The largest standard error for the results in column 6 is 2\%. The largest standard error for the results in column 7 is 2\%.
    \end{tablenotes}
  \end{threeparttable}
\end{table}

\begin{table}[H]
\centering
  \begin{threeparttable}
    \caption{Low dimensional coverage simulation with random forest}
     \begin{tabular}{lcccccc}
       $v$& $\textsc{CATE}(v)$ & Tuning &  Ave. Est. & Ave. S.E. &  80\% Cov. & 95\% Cov. \\[5pt]
-0.25 & -0.10 & 0.25 & -0.10 & 0.05 &  86\% & 93\% \\
-0.25 & -0.10 & 0.50 & -0.09 & 0.03  & 83\% & 94\% \\
-0.25 & -0.10 & 1.00 & -0.08 & 0.02  & 60\% & 79\% \\
0.00 & 0.00 & 0.25 & 0.00 & 0.02 &  70\% & 91\% \\
0.00 & 0.00 & 0.50 & 0.01 & 0.02  & 72\% & 91\% \\
0.00 & 0.00 & 1.00 & 0.02 & 0.02  & 48\% & 75\% \\
0.25 & 0.32 & 0.25 & 0.30 & 0.12 &  83\% & 91\% \\
0.25 & 0.32 & 0.50 & 0.29 & 0.08  & 82\% & 91\% \\
0.25 & 0.32 & 1.00 & 0.28 & 0.06  & 71\% & 86\% 
     \end{tabular}
     \label{tab:rf_low}
    \begin{tablenotes}
      \small
      \item Ave., average; Est., estimate; S.E., standard error; Cov., coverage. The largest standard error for the results in column 6 is 2\%. The largest standard error for the results in column 7 is 2\%.
    \end{tablenotes}
  \end{threeparttable}
\end{table}

\begin{table}[H]
\centering
  \begin{threeparttable}
    \caption{Low dimensional coverage simulation with lasso}
     \begin{tabular}{lcccccc}
       $v$& $\textsc{CATE}(v)$ & Tuning &  Ave. Est. & Ave. S.E. &  80\% Cov. & 95\% Cov. \\[5pt]
-0.25 & -0.10 & 0.25 & -0.08 & 0.08 &  81\% & 95\% \\
-0.25 & -0.10 & 0.50 & -0.08 & 0.05  & 81\% & 95\% \\
-0.25 & -0.10 & 1.00 & -0.06 & 0.04  & 63\% & 88\% \\
0.00 & 0.00 & 0.25 & 0.00 & 0.06 &  79\% & 94\% \\
0.00 & 0.00 & 0.50 & 0.01 & 0.04  & 83\% & 96\% \\
0.00 & 0.00 & 1.00 & 0.02 & 0.03  & 73\% & 92\% \\
0.25 & 0.32 & 0.25 & 0.30 & 0.11 &  86\% & 94\% \\
0.25 & 0.32 & 0.50 & 0.29 & 0.08  & 85\% & 95\% \\
0.25 & 0.32 & 1.00 & 0.28 & 0.06  & 71\% & 89\% 
     \end{tabular}
     \label{tab:lasso_low}
    \begin{tablenotes}
      \small
      \item Ave., average; Est., estimate; S.E., standard error; Cov., coverage. The largest standard error for the results in column 6 is 2\%. The largest standard error for the results in column 7 is 1\%.
    \end{tablenotes}
  \end{threeparttable}
\end{table}

For the high dimensional regime, Tables~\ref{tab:nn_high},~\ref{tab:rf_high}, and~\ref{tab:lasso_high} summarize results for neural network, random forest, and lasso, respectively. With bandwidth hyperparameter values $c_h=0.25$ and $c_h=0.50$, coverage is close to the nominal level across $\hat{\gamma}$ estimators and for grid values $v=-0.25$ and $v=0.00$. Across $\hat{\gamma}$ estimators, the grid value $v=0.25$ is more challenging. Compared to the low dimensional regime, each estimator in the high dimensional regime has higher bias and compensates with higher variance for the grid value $v=0.25$.

\begin{table}[H]
\centering
  \begin{threeparttable}
    \caption{High dimensional coverage simulation with neural network}
     \begin{tabular}{lcccccc}
       $v$& $\textsc{CATE}(v)$ & Tuning &  Ave. Est. & Ave. S.E. &  80\% Cov. & 95\% Cov. \\[5pt]
-0.25 & -0.10 & 0.25 & -0.09 & 0.04 &  84\% & 91\% \\
-0.25 & -0.10 & 0.50 & -0.09 & 0.03  & 78\% & 91\% \\
-0.25 & -0.10 & 1.00 & -0.07 & 0.02  & 49\% & 74\% \\
0.00 & 0.00 & 0.25 & 0.00 & 0.03 &  75\% & 95\% \\
0.00 & 0.00 & 0.50 & 0.01 & 0.02  & 74\% & 91\% \\
0.00 & 0.00 & 1.00 & 0.04 & 0.03  & 52\% & 75\% \\
0.25 & 0.32 & 0.25 & 0.39 & 0.20 &  90\% & 97\% \\
0.25 & 0.32 & 0.50 & 0.39 & 0.15  & 88\% & 97\% \\
0.25 & 0.32 & 1.00 & 0.38 & 0.13  & 81\% & 95\% 
     \end{tabular}
     \label{tab:nn_high}
    \begin{tablenotes}
      \small
      \item Ave., average; Est., estimate; S.E., standard error; Cov., coverage. The largest standard error for the results in column 6 is 2\%. The largest standard error for the results in column 7 is 2\%.
    \end{tablenotes}
  \end{threeparttable}
\end{table}

\begin{table}[H]
\centering
  \begin{threeparttable}
    \caption{High dimensional coverage simulation with random forest}
     \begin{tabular}{lcccccc}
       $v$& $\textsc{CATE}(v)$ & Tuning &  Ave. Est. & Ave. S.E. &  80\% Cov. & 95\% Cov. \\[5pt]
-0.25 & -0.10 & 0.25 & -0.09 & 0.05 &  81\% & 91\% \\
-0.25 & -0.10 & 0.50 & -0.09 & 0.03  & 78\% & 91\% \\
-0.25 & -0.10 & 1.00 & -0.07 & 0.02  & 53\% & 75\% \\
0.00 & 0.00 & 0.25 & 0.00 & 0.02 &  76\% & 94\% \\
0.00 & 0.00 & 0.50 & 0.01 & 0.02  & 74\% & 92\% \\
0.00 & 0.00 & 1.00 & 0.04 & 0.03  & 44\% & 71\% \\
0.25 & 0.32 & 0.25 & 0.37 & 0.18 &  91\% & 96\% \\
0.25 & 0.32 & 0.50 & 0.40 & 0.16  & 88\% & 97\% \\
0.25 & 0.32 & 1.00 & 0.39 & 0.14  & 81\% & 95\% 
     \end{tabular}
     \label{tab:rf_high}
    \begin{tablenotes}
      \small
      \item Ave., average; Est., estimate; S.E., standard error; Cov., coverage. The largest standard error for the results in column 6 is 2\%. The largest standard error for the results in column 7 is 2\%.
    \end{tablenotes}
  \end{threeparttable}
\end{table}

\begin{table}[H]
\centering
  \begin{threeparttable}
    \caption{High dimensional coverage simulation with lasso}
     \begin{tabular}{lcccccc}
       $v$& $\textsc{CATE}(v)$ & Tuning &  Ave. Est. & Ave. S.E. &  80\% Cov. & 95\% Cov. \\[5pt]
-0.25 & -0.10 & 0.25 & -0.08 & 0.06 &  75\% & 90\% \\
-0.25 & -0.10 & 0.50 & -0.07 & 0.04  & 74\% & 90\% \\
-0.25 & -0.10 & 1.00 & -0.06 & 0.03  & 50\% & 74\% \\
0.00 & 0.00 & 0.25 & 0.01 & 0.05 &  78\% & 96\% \\
0.00 & 0.00 & 0.50 & 0.02 & 0.04  & 80\% & 97\% \\
0.00 & 0.00 & 1.00 & 0.04 & 0.04  & 59\% & 84\% \\
0.25 & 0.32 & 0.25 & 0.41 & 0.20 &  89\% & 96\% \\
0.25 & 0.32 & 0.50 & 0.45 & 0.18  & 87\% & 97\% \\
0.25 & 0.32 & 1.00 & 0.43 & 0.17  & 82\% & 95\% 
     \end{tabular}
     \label{tab:lasso_high}
    \begin{tablenotes}
      \small
      \item Ave., average; Est., estimate; S.E., standard error; Cov., coverage. The largest standard error for the results in column 6 is 2\%. The largest standard error for the results in column 7 is 2\%.
    \end{tablenotes}
  \end{threeparttable}
\end{table}

%% file: B_local.tex
\section{Characterization of key parameters}\label{sec:localization}

\subsection{Additional example}

We present an additional example useful in commercial applications where an analyst has access to data on household spending behavior. See \cite{chernozhukov2019demand} for a motivating economic model and further interpretation.

\begin{example}[Heterogeneous average derivative estimated by lasso]\label{ex:deriv}
Let $Y$ be share of household expenditure on some good. Let $W=(D,V,X)$ concatenate log price of the good $D$, covariate of interest $V$ such as household size, and other covariates $X$ such as log prices of other goods and log total household expenditure. Let $\gamma_0(d,v,x)=E(Y \mid D=d,V=v,X=x)$ be a function estimated by lasso. The heterogeneous average derivative is
$$
\textsc{DERIV}(v)=E\left\{\frac{\partial }{\partial d} \gamma_0(D,V,X) \mid V=v\right\}=\lim_{h\rightarrow 0}E\left\{\ell_h(V)\frac{\partial }{\partial d} \gamma_0(D,V,X)\right\},$$
where
$
\ell_h(V)=(h\omega)^{-1}K\left\{(V-v)/h\right\}$, $\omega=E [h^{-1} K\left\{(V-v)/h\right\}]
$, and $K$ is a bounded and symmetric kernel that integrates to one.
\end{example}
The heterogeneous average derivative is defined with respect to some interpretable, low dimensional characteristic $V$ such as household size \cite{abrevaya2015estimating}. The same functional without the localization $\ell_h$ is the classic average derivative. 

\subsection{Riesz representers}

We begin by deriving explicit expressions for Riesz representer $\alpha_0$. Recall that the unique minimal Riesz representer $\alpha_0^{\min}$ is the projection of any valid Riesz representer $\alpha_0$ onto $\Gamma$. The explicit expressions for $\alpha_0$ help to articulate simple sufficient conditions for existence of $\alpha_0^{\min}$.

\begin{lemma}\label{lemma:RR_exists}
In Examples~\ref{ex:CATE},~\ref{ex:RDD},~\ref{ex:elasticity}, and~\ref{ex:deriv}, the minimal representer $\alpha_0^{\min}(w)$ can be obtained by projecting the following Riesz representer $\alpha_0(w)$ onto $\Gamma$.
\begin{enumerate}
    \item Example~\ref{ex:CATE}. Denote the propensity score $\pi_0(v,x)=\text{\normalfont pr}(D=1\mid V=v,X=x)$. Then
    $$
    \alpha_0(d,v,x)=\ell_h(v)\left\{\frac{d}{\pi_0(v,x)}-\frac{1-d}{1-\pi_0(v,x)}\right\}.
    $$
    Hence the minimal Riesz representer $\alpha^{\min}_0$ exists if $\pi_0(v,x)$ is bounded away from zero and one.
    \item Example~\ref{ex:RDD}. Denote the Riesz representer for the first term by $\alpha_0^{+}$ and the Riesz representer for the second term by $\alpha_0^{-}$. Then
    $$
    \alpha^+_0(d,x)=\ell_h^{+}(d),
    \quad 
     \alpha^-_0(d,x)=\ell_h^{-}(d).
    $$
    Hence the minimal Riesz representer $\alpha^{\min}_0$ exists.
    \item Example~\ref{ex:elasticity}. Denote the density $f(d,x)$. If $f(d\mid x)$ vanishes for each $d$ in the boundary of the support of $D$ given $X=x$ almost everywhere then
    $$
    \eta_0(d,x)
    =-\partial_d \log f(d \mid x).$$ Hence the minimal representer $\eta_0^{\min}$ exists if $-\partial_d \log f(d \mid x)$ is bounded above. Subsequently, $\alpha_0^{\min}$ is the solution $\alpha$ to
    $$
    \eta_0^{\min}(d,x)=E\{\alpha(X,Z) \mid D=d,X=x\}.
    $$
    \item Example~\ref{ex:deriv}. Denote the density $f(d,v,x)$. If $f(d\mid v,x)$ vanishes for each $d$ in the boundary of the support of $D$ given $(V,X)=(v,x)$ almost everywhere then
    $$
    \alpha_0(d,v,x)
    =-\ell_h(v)\partial_d \log f(d \mid v,x).$$
\end{enumerate}
\end{lemma}

Next, we characterize key quantities $(\bar{Q},\bar{\sigma},\bar{\alpha},\bar{\alpha}')$ and moments $(\sigma,\kappa,\zeta)$ that appear in Theorems~\ref{thm:dml} and~\ref{thm:var}. Recall that
$$
E[\{Y-\gamma_0(W)\}^2 \mid W]\leq \bar{\sigma}^2, \quad  \|\hat{\alpha}_{\ell}\|_{\infty}\leq\bar{\alpha}',
$$
so $\bar{\sigma}$ is a constant if noise has bounded variance and $\bar{\alpha}'$ can be imposed by trimming the algorithm $\hat{\alpha}_{\ell}$. We therefore focus on $(\bar{Q},\bar{\alpha})$ and $(\sigma,\kappa,\zeta)$.

We consider two cases: global functionals, with weighting $\ell_h$ that has a fixed bandwidth; and local functionals, with weighting $\ell_h$ that depends on some vanishing bandwidth $h\rightarrow 0$. To lighten notation, let $\|X\|_{\text{pr},q}=\{E(|X|^q)\}^{1/q}$.

\subsection{Global functionals}

\begin{lemma}[Oracle moments for global functionals]\label{lemma:global}
Suppose the weighting is bounded,  namely $\ell_h \leq C<\infty$, and $\Gamma\subset \mathbb{L}_2$. Suppose $\alpha_0^{\min}$ exists. Further suppose there exist $(c,\tilde{c},\bar c)$ bounded away from zero and above such that the following conditions hold.
\begin{enumerate}
    \item Control of representer moments. For $q=3,4$
    $$
\|\alpha_0^{\min}\|_{\text{\normalfont pr},q}\leq c \left( \|\alpha_0^{\min}\|^2_{\text{\normalfont pr},2} \vee 1\right).
$$
    \item Bounded moments. For $q=2,3,4$,
    $$
    U_1=m(W,\gamma_0)-E\{m(W,\gamma_0)\},\quad \|U_1\|_{\text{\normalfont pr},q}\leq \bar{c}.
    $$
    \item Bounded heteroscedasticity. For $q=2,3,4$,
    $$
    U_2=Y-\gamma_0(W),\quad \tilde{c}\leq \|U_2 \mid W\|_{\text{\normalfont pr},q}\leq \bar{c}.
    $$
\end{enumerate}
Then
    $$
    \tilde{c}  \bar{M}   \leq \sigma  \leq  \bar c  \sqrt{1+ \bar{M}^2}, 
 \quad \kappa,\zeta  \leq  \bar c   \{1+ c (\bar{M}^2 \vee 1) \}.
    $$
In summary,
$$
\frac{\kappa}{\sigma} \lesssim \sigma \asymp \bar{M} < \infty,\quad \kappa, \zeta\lesssim \bar{M}^2 < \infty.
$$
\end{lemma}

Clearly Assumption~\ref{assumption:cont} depends on the functional of interest. Towards a characterization of $\bar{Q}$ and $q$ in Examples~\ref{ex:elasticity} and~\ref{ex:deriv}, we prove the following technical lemma.

\begin{lemma}[A weak reverse Poincare inequality]\label{lemma:rp}
Assume that $f(d\mid x)$ vanishes for each $d$ in the boundary of the support of $D$ given $X=x$ almost everywhere. Next assume the following restrictions on $\Gamma \subset \mathbb{L}_2$.
\begin{enumerate}
    \item Each $\gamma$ in $\Gamma$ is twice continuously differentiable.
    \item For each $\gamma$ in $\Gamma$, $\|k_{\gamma}\|_{\text{ \normalfont pr},2}<\infty$ where
    $$
    k_{\gamma}(d,x)=\{-\partial_d \log f(d\mid x)\}\{\partial_d\gamma(d,x)\}-\partial_d^2 \gamma(d,x).
    $$
\end{enumerate}
Then 
$$
E[\{\partial_d\gamma(D,X)\}^2]\leq \|k_{\gamma}\|_{\text{ \normalfont pr},2} [E \{\gamma(D,X)^2\}]^{1/2}.
$$
Furthermore, $\sup_{\gamma \in \Gamma} \|k_{\gamma}\|_{\text{ \normalfont pr},2}<\infty$ if either of the following conditions hold.
\begin{enumerate}
    \item $\|\partial_d \log f(D\mid X)\|_{\text{ \normalfont pr},2}<\infty$ and for all $\gamma$ in $\Gamma$, $\|\partial_d \gamma\|_{\infty}<\infty$ and $\|\partial^2_d \gamma\|_{\text{ \normalfont pr},2}<\infty$;
    \item $-\partial_d \log f(d\mid x)$ is bounded above and for all $\gamma$ in $\Gamma$,  $\|\partial_d \gamma\|_{\text{ \normalfont pr},2}$ and $\|\partial^2_d \gamma\|_{\text{ \normalfont pr},2}<\infty$.
\end{enumerate}
\end{lemma}

With this technical lemma, we return to the characterization of $\bar{Q}$ and $q$ across examples.

\begin{lemma}[Mean square continuity for global functionals]\label{lemma:cont}
Suppose the weighting is bounded,  namely $\ell_h \leq C<\infty$, and $\Gamma\subset \mathbb{L}_2$. The following conditions are sufficient for $\bar{Q}<\infty$ with $q=1$ in Examples~\ref{ex:CATE} and~\ref{ex:RDD}, and for $\bar{Q}<\infty$ with $q=1/2$ in Examples~\ref{ex:elasticity} and~\ref{ex:deriv}.
\begin{enumerate}
    \item Example~\ref{ex:CATE}. $\pi_0(v,x)$ is bounded away from zero and one.
    \item Example~\ref{ex:RDD}. The bandwidth is fixed.
    \item Example~\ref{ex:elasticity}. $f(d\mid x)$ vanishes for each $d$ in the boundary of the support of $D$ given $X=x$ almost everywhere. $-\partial_d \log f(d \mid x)$ is a bounded above, and $\Gamma$ consists of functions
    $\gamma$ that are twice continuously differentiable in the first argument and that satisfy $E[ \{\partial_d \gamma(D,X)\}^2]<\infty$ and $E[ \{\partial^2_d \gamma(D,X)\}^2]<\infty$.
     \item Example~\ref{ex:deriv}. $f(d\mid v,x)$ vanishes for each $d$ in the boundary of the support of $D$ given $(V,X)=(v,x)$ almost everywhere. $-\partial_d \log f(d \mid v,x)$ is a bounded above, and $\Gamma$ consists of functions
    $\gamma$ that are twice continuously differentiable in the first argument and that satisfy $E[ \{\partial_d \gamma(D,V,X)\}^2]<\infty$ and $E[ \{\partial^2_d \gamma(D,V,X)\}^2]<\infty$.
\end{enumerate}
\end{lemma}
Therefore a Sobolev type property with respect to the first argument is a sufficient condition in Examples~\ref{ex:elasticity} and~\ref{ex:deriv}.

Next we examine the assumption of $\|\alpha^{\min}_0\|_{\infty} \leq \bar{\alpha}$, which depends on the functional of interest.
\begin{lemma}[Bounded Riesz representer for global functionals]\label{lemma:bounded_RR_global}
The following conditions are sufficient for $\bar{\alpha}<\infty$ in Examples~\ref{ex:CATE},~\ref{ex:RDD},~\ref{ex:elasticity}, and~\ref{ex:deriv} with $\ell_h \leq C<\infty$ and $\Gamma=\mathbb{L}_2$.
\begin{enumerate}
    \item Example~\ref{ex:CATE}. $\pi_0(v,x)$ is bounded away from zero and one.
    \item Example~\ref{ex:RDD}. The bandwidth is fixed.
    \item Example~\ref{ex:elasticity}. $\alpha_0^{\min}$ that solves
    $
    -\partial_d \log f(d \mid x)=E\{\alpha(X,Z) \mid D=d,X=x\}
    $ is bounded above.
    \item Example~\ref{ex:deriv}. $
    -\partial_d \log f(d \mid v,x)
    $ is bounded above.
\end{enumerate}
\end{lemma}

\subsection{Local functionals}

Given a local functional
$$
\theta_0^h=E\{m_h(W,\gamma_0)\}=E\{\ell_h(V) m(W,\gamma_0)\},
$$
we will now refer to the Riesz representer of $m_h$ by $\alpha^h_0$ and the Riesz representer of $m$ by $\alpha_0$ for this subsection. The latter objects correspond to the global setting where the weighting is bounded. To lighten notation, we write $\ell=\ell_h$.

\begin{lemma}[Oracle moments for local functionals]\label{lemma:local}
Suppose $\alpha_0^{\min}$ exists and $\Gamma=\mathbb{L}_2$. Further suppose there exist $(\tilde{\alpha}, \check{\alpha}, \tilde{c}, \bar c, \tilde{f}, \bar f, \bar f', h_0)$ bounded away from zero and above such that the following conditions hold.
\begin{enumerate}
    \item Control of representer absolute value:
    $$
0< \tilde{\alpha} \leq \alpha_0(w) \leq \check{\alpha}.
$$
  \item Valid neighborhood. There exists $N_{h_0}(v)=(v':|v'-v|\leq h)\subset \mathcal{V}$.
    \item Bounded moments. For all $ h\leq h_0$ and for $q=2,3,4$, 
    $$
    U_1=m_h(W,\gamma_0)-E\{m_h(W,\gamma_0)\},\quad \|U_1\|_{\text{\normalfont pr},q}\leq \bar{c}\|\ell\|_{\text{\normalfont pr},q}.
    $$
    \item Bounded heteroscedasticity. For $q=2,3,4$,
    $$
    U_2=Y-\gamma_0(W),\quad \tilde{c}\leq \|U_2 \mid W\|_{\text{\normalfont pr},q}\leq \bar{c}.
    $$
    \item Bounded density. The density $f_V$ obeys, for all $v'$ in $N_{h_0}(v)$,
    $$
  0< \tilde{f} \leq f_V(v') \leq \bar f,\quad |\partial f_V(v')| \leq \bar f'.
    $$
\end{enumerate}
Then finite sample bounds in the proof hold. In summary,
$$
\frac{\kappa_h}{\sigma_h} \lesssim h^{-1/6} \lesssim \sigma_h \asymp \bar{M}_h \asymp h^{-1/2} \rightarrow \infty,\quad \kappa_h\lesssim h^{-2/3} \rightarrow \infty,\quad \zeta_h\lesssim h^{-3/4} \rightarrow \infty.
$$

\end{lemma}

The conditions of Lemma~\ref{lemma:bounded_RR_global} suffice for $0< \tilde{\alpha} \leq \alpha_0(w) \leq \check{\alpha}$ in Examples~\ref{ex:CATE} and~\ref{ex:RDD}.

As before, Assumption~\ref{assumption:cont} depends on the functional of interest.
\begin{lemma}[Mean square continuity for local functionals]\label{lemma:cont_local}
Suppose $\alpha_0^{\min}$ exists and $\Gamma\subset \mathbb{L}_2$. Further suppose there exist $(\tilde{f}, \bar f, \bar f', h_0)$ bounded away from zero and above such that the following conditions hold.
\begin{enumerate}
  \item Valid neighborhood. There exists $N_{h_0}(v)=(v':|v'-v|\leq h)\subset \mathcal{V}$.
    \item Bounded density. The density $f_V$ obeys, for all $v'$ in $N_{h_0}(v)$,
    $$
  0< \tilde{f} \leq f_V(v') \leq \bar f,\quad |\partial f_V(v')| \leq \bar f'.
    $$
    \item The conditions of Lemma~\ref{lemma:cont} hold.
\end{enumerate}
Then the finite sample bound in the proof holds for Examples~\ref{ex:CATE} and~\ref{ex:RDD}. In summary,
$$
\bar{Q}_h\lesssim h^{-2} \rightarrow \infty.
$$
\end{lemma}

As before, the assumption of $\|\alpha^{\min,h}_0\|_{\infty} \leq \bar{\alpha}$ depends on the functional of interest.
\begin{lemma}[Bounded Riesz representer for local functionals]\label{lemma:bounded_RR_local}
Suppose $\alpha_0^{\min}$ exists and $\Gamma\subset \mathbb{L}_2$. Further suppose there exist $(\tilde{f}, \bar f, \bar f', h_0,\bar{K})$ bounded away from zero and above such that the following conditions hold.
\begin{enumerate}
  \item Valid neighborhood. There exists $N_{h_0}(v)=(v':|v'-v|\leq h)\subset \mathcal{V}$.
    \item Bounded density. The density $f_V$ obeys, for all $v'$ in $N_{h_0}(v)$,
    $$
  0< \tilde{f} \leq f_V(v') \leq \bar f,\quad |\partial f_V(v')| \leq \bar f'.
    $$
    \item Bounded kernel. $|K(u)|\leq \bar{K}$.
    \item The conditions of Lemma~\ref{lemma:bounded_RR_global} hold, allowing bandwidth to vanish.
\end{enumerate}
Then the finite sample bound in the proof holds. In summary,
$$
\bar{\alpha}_h\lesssim h^{-1}\rightarrow \infty.
$$
\end{lemma}

The main results are in terms of abstract mean square rates $\mathcal{R}(\hat{\alpha}^h_{\ell})$ and $\mathcal{P}(\hat{\alpha}^h_{\ell})$ for the local Riesz representer $\alpha_0^{\min,h}$ of the functional $m_h$. A growing literature proposes machine learning estimators $\hat{\alpha}$ with rates $\mathcal{R}(\hat{\alpha}_{\ell})$ and $\mathcal{P}(\hat{\alpha}_{\ell})$ for the global Riesz representer $\alpha^{\min}_0$ of the functional $m$.

A natural choice of estimator $\hat{\alpha}^h$ for $\alpha_0^{\min,h}$ is the localization $\ell_h$ times an estimator $\hat{\alpha}$ for $\alpha^{\min}_0$. We prove that this choice allows an analyst to translate global Riesz representer rates into local Riesz representer rates under mild regularity conditions. In Supplement~1, we confirm that this choice performs well in simulations.

\begin{lemma}[Translating global rates to local rates]\label{lemma:translate_RR}
Suppose the conditions of Lemma~\ref{lemma:bounded_RR_local} hold with $\Gamma=\mathbb{L}_2$. Then
$$
\mathcal{R}(\hat{\alpha}^h_{\ell}) \lesssim h^{-2} \mathcal{R}(\hat{\alpha}_{\ell}),\quad \mathcal{P}(\hat{\alpha}^h_{\ell}) \lesssim h^{-2} \mathcal{P}(\hat{\alpha}_{\ell}).
$$
\end{lemma}

\subsection{Approximation error}

Finally, we characterize the finite sample approximation error $\Delta_h=
n^{1/2} \sigma^{-1}|\theta_0^h-\theta_0^{\lim}|$ where 
$$
\theta^{\lim}_{0}=\lim_{h\rightarrow 0} \theta_0^h,\quad \theta_0^h=E\{m_h(W,\gamma_0)\}=E\{\ell_h(V) m(W,\gamma_0)\}.
$$
$\Delta_h$ is bias from using a semiparametric sequence to approximate a nonparametric quantity.

We define $m(v)= E [m(W, \gamma_0) \mid V=v]$ to lighten notation.

\begin{lemma}[Approximation error from localization \cite{chernozhukov2018global}]\label{lemma:approx}
Suppose there exist constants $(h_0,K, \mathsf{v}, \bar g_{\mathsf{v}}$,  $\bar f_{\mathsf{v}}$, $\tilde{f}, \bar{g})$ bounded away from zero and above such that the following conditions hold.
\begin{enumerate}
  \item Valid neighborhood. There exists $N_{h_0}(v)=(v':|v'-v|\leq h)\subset \mathcal{V}$.
    \item Differentiability. On $N_{h_0}(v)$, $m(v')$ and $f_V(v')$ are differentiable to the integer order $\mathsf{sm}$.
    \item Bounded derivatives. Let $\mathsf{v}= \mathsf{sm} \wedge \mathsf{o}$ where $\mathsf{o}$ is the order of the kernel $K$. Let $\partial^\mathsf{v}_d$ denote the $\mathsf{v}$ order derivative $\partial^{\mathsf{v}}/(\partial d)^{\mathsf{v}}$. Assume
    $$
\sup_{ v' \in N_{h_0}(v)}  \| \partial^{\mathsf{v}}_v (m(v') f_V(v')) \|_{op} \leq \bar g_{\mathsf{v}}, \quad \sup_{ v' \in N_{h_0}(v) } \|\partial^{\mathsf{v}}_v f_V(v')  \|_{op} \leq \bar f_{\mathsf{v}},
 \quad  \inf_{ v' \in N_{h_0}(v) } f_V(v') \geq \tilde{f}.
$$ 
\item Bounded conditional formula.
$
m(v)f_V(v)\leq \bar{g}.
$
\end{enumerate}
Then there exist constants $(C,h_1)$ depending only on $(h_0,K, \mathsf{v}, \bar g_{\mathsf{v}}$,  $\bar f_{\mathsf{v}}$, $\tilde{f}, \bar{g})$ such that for all $h_1$ in $(h,h_0)$, $|\theta_0^h-\theta_0^{\lim}|\leq C h^\mathsf{v}.$
In summary,
$$
\Delta_h \lesssim n^{1/2} h^{\mathsf{v}+1/2}.
$$
\end{lemma}

To summarize the characterizations in this Supplement, we provide a corollary for local functionals. Let $\mathcal{R}(\hat{\alpha}_{\ell})$ and $\mathcal{P}(\hat{\alpha}_{\ell})$ be defined as in Lemma~\ref{lemma:translate_RR}.

\begin{corollary}[Confidence interval for local functionals]\label{cor:CI_local}
Suppose the conditions of Corollary~\ref{cor:CI} and Lemmas~\ref{lemma:local},~\ref{lemma:cont_local},~\ref{lemma:bounded_RR_local},~\ref{lemma:translate_RR}, and~\ref{lemma:approx} hold. As $n\rightarrow \infty$ and $h\rightarrow 0$, suppose the regularity condition on  moments $n^{-1/2}h^{-3/2}\rightarrow 0$ as well as the following learning rate conditions:
\begin{enumerate}
    \item $\left(h^{-1}+\bar{\alpha}'\right)\{\mathcal{R}(\hat{\gamma}_{\ell})\}^{q/2}=o_p(1)$;
    \item $\bar{\sigma}h^{-1}\{\mathcal{R}(\hat{\alpha}_{\ell})\}^{1/2}=o_p(1)$;
    \item $h^{-1/2}[\{n \mathcal{R}(\hat{\gamma}_{\ell}) \mathcal{R}(\hat{\alpha}_{\ell})\}^{1/2} \wedge \{n\mathcal{P}(\hat{\gamma}_{\ell})\mathcal{R}(\hat{\alpha}_{\ell})\}^{1/2} \wedge \{n\mathcal{R}(\hat{\gamma}_{\ell})\mathcal{P}(\hat{\alpha}_{\ell})\}^{1/2}] =o_p(1)$.
\end{enumerate}
Finally suppose the approximation error condition $\Delta_h \rightarrow 0$. Then the estimator $\hat{\theta}^h$ in Algorithm~\ref{alg:target} is consistent and asymptotically Gaussian, and the confidence interval in Algorithm~\ref{alg:target} includes $\theta^{\lim}_0$ with probability approaching the nominal level. Formally,
$$
\hat{\theta}^h=\theta^{\lim}_0+o_p(1),\quad \sigma_h^{-1}n^{1/2}(\hat{\theta}^h-\theta^{\lim}_0)\leadsto \mathcal{N}(0,1),\quad  \text{\normalfont pr} \left\{\theta^{\lim}_0 \text{ in }  \left(\hat{\theta}^h\pm c_a\hat{\sigma} n^{-1/2} \right)\right\}\rightarrow 1-a.
$$
\end{corollary}

%% file: C_discussion.tex
\section{Discussion}

 In independent work, \cite[Theorem 9]{kallus2021causal} present an asymptotic Gaussian approximation result for a particular global functional: average treatment effect identified by negative controls. This functional fits within our framework because it is a mean square continuous functional of a nonparametric instrumental variable regression. To verify mean square continuity, see Example~\ref{ex:CATE} in Lemma~\ref{lemma:cont}.
 
 In their analysis, the authors write the sufficient condition
    $$\min(\tau_{\gamma,n},\tau_{\alpha,n}) \iota_{\gamma,n} \iota_{\alpha,n}=o(n^{-1/2}),\quad  \tau_{\gamma,n}=\sup_{\gamma \in \mathcal{G}_n}\frac{\{\mathcal{R}(\gamma)\}^{1/2}}{\{\mathcal{P}(\gamma)\}^{1/2}},\quad \tau_{\alpha,n}=\sup_{\alpha \in \mathcal{A}_n}\frac{\{\mathcal{R}(\alpha)\}^{1/2}}{\{\mathcal{P}(\alpha)\}^{1/2}}$$
    where
    $
   (\tau_{\gamma,n},\tau_{\alpha,n})
    $
    are ratio measures of ill posedness and $(\iota_{\gamma,n},\iota_{\alpha,n})$ are critical radii for the sequence of function classes $(\mathcal{G}_n,\mathcal{A}_n)$ used in adversarial estimation procedures for $(\hat{\gamma},\hat{\alpha})$. In particular $(\iota_{\gamma,n},\iota_{\alpha,n})$ appear in the authors' bounds for $\{\mathcal{P}(\hat{\gamma})\}^{1/2}$ and $\{\mathcal{P}(\hat{\alpha})\}^{1/2}$, respectively.
    
    For comparison, our analogous condition in Theorem~\ref{thm:dml} is that
    $$
   \min\left[\{\mathcal{P}(\hat{\gamma}_{\ell})\mathcal{R}(\hat{\alpha}_{\ell})\}^{1/2}, \{\mathcal{R}(\hat{\gamma}_{\ell})\mathcal{P}(\hat{\alpha}_{\ell})\}^{1/2}\right]=o(\sigma n^{-1/2}).
    $$
    By contrast, our result is (i) for the entire class of mean square continuous functionals, (ii) for black box estimators $(\hat{\gamma},\hat{\alpha})$, and (iii) finite sample, so it also handles local functionals in which $\sigma$ diverges. Critical radius arguments are one way to prove mean square rates for certain machine learning estimators, but not the only way. This distinction is important, since many existing statistical learning theory rates, whether for neural networks as in Example~\ref{ex:CATE}, random forest as in Example~\ref{ex:RDD}, kernel ridge regression as in  Example~\ref{ex:elasticity}, or lasso as in Example~\ref{ex:deriv}, are not in terms of a critical radius.

%% file: D_proofs.tex
\section{Proof of main result}

\subsection{Gateaux differentiation}

Recall the notation
$$
 \psi_0(w)=\psi(w,\theta_0,\gamma_0,\alpha^{\min}_0), \quad \psi(w,\theta,\gamma,\alpha)=m(w,\gamma)+\alpha(w)\{y-\gamma(w)\}-\theta,
$$
where $\gamma\mapsto m(w,\gamma)$ is linear. For readability, we introduce the following notation for Gateaux differentiation. 
\begin{definition}[Gateaux derivative]
Let $u(w),v(w)$ be functions and let $\tau,\zeta \text{ in }  \mathbb{R}$ be scalars. The Gateaux derivative of $\psi(w,\theta,\gamma,\alpha)$ with respect to its argument $\gamma$ in the direction $u$ is
$$
\{\partial_{\gamma} \psi(w,\theta,\gamma,\alpha)\}(u)=\frac{\partial}{\partial \tau} \psi(w,\theta,\gamma+\tau u,\alpha) \bigg|_{\tau=0}.
$$
The cross derivative of $\psi(w,\theta,\gamma,\alpha)$ with respect to its arguments $(\gamma,\alpha)$ in the directions $(u,v)$ is
$$
\{\partial^2_{\gamma,\alpha} \psi(w,\theta,\gamma,\alpha)\}(u,v)=\frac{\partial^2}{\partial \tau \partial \zeta} \psi(w,\theta,\gamma+\tau u,\alpha+\zeta v) \bigg|_{\tau=0,\zeta=0}.
$$
\end{definition}

\begin{proposition}[Calculation of derivatives]\label{prop:deriv}
\begin{align*}
    \{\partial_{\gamma} \psi(w,\theta,\gamma,\alpha)\}(u)&=m(w,u)-\alpha(w)u(w); \\
    \{\partial_{\alpha} \psi(w,\theta,\gamma,\alpha)\}(v)&=v(w)\{y-\gamma(w)\}; \\
    \{\partial^2_{\gamma,\alpha} \psi(w,\theta,\gamma,\alpha)\}(u,v)&=-v(w)u(w).
\end{align*}
\end{proposition}

\begin{proof}
For the first result, write
$$
\psi(w,\theta,\gamma+\tau u,\alpha)=m(w,\gamma)+\tau m(w,u)+\alpha(w)\{y-\gamma(w)-\tau u(w)\}-\theta.
$$
For the second result, write
$$
\psi(w,\theta,\gamma,\alpha+\zeta v)=m(w,\gamma)+\alpha(w)\{y-\gamma(w)\}+\zeta v(w)\{y-\gamma(w)\}-\theta.
$$
For the final result, write
\begin{align*}
    &\psi(w,\theta,\gamma+\tau u,\alpha+\zeta v)\\
    &=m(w,\gamma)+\tau m(w,u)+\alpha(w)\{y-\gamma(w)-\tau u(w)\}+\zeta v(w)\{y-\gamma(w)-\tau u(w)\}-\theta.
\end{align*}
Finally, take scalar derivatives with respect to $(\tau,\zeta)$.
\end{proof}

By using the doubly robust moment function, we have the following helpful property.
\begin{proposition}[Mean zero derivatives]\label{prop:mean_zero}
For any $(u,v)$,
$$
E\{\partial_{\gamma} \psi_0(W)\}(u)=0,\quad E\{\partial_{\alpha} \psi_0(W)\}(v)=0.
$$
\end{proposition}

\begin{proof}
We appeal to Proposition~\ref{prop:deriv}. For the first result, write
$$
E\{\partial_{\gamma} \psi_0(W)\}(u)=E\{m(W,u)-\alpha_0^{\min}(W)u(W)\}.
$$
In the case of nonparametric regression or projection, appeal to the definition of minimal Riesz representer $\alpha_0^{\min}$. In the case of nonparametric instrumental variable regression,
\begin{align*}
    E\{m(W_1,u)-\alpha_0^{\min}(W_2)u(W_1)\}
    &=E[\{\eta_0^{\min}(W_1)-\alpha_0^{\min}(W_2)\}u(W_1)] \\
    &=E([\eta_0^{\min}(W_1)-E\{\alpha_0^{\min}(W_2) \mid W_1\}]u(W_1)) \\
    &=0.
\end{align*}
For the second result, write
$$
E\{\partial_{\alpha} \psi_0(W)\}(v)=E[v(W)\{Y-\gamma_0(W)\}].
$$
In the case of nonparametric regression, $\gamma_0(w)=E(Y\mid W=w)$ and we appeal to law of iterated expectations. In the case of nonparametric projection, the desired result holds by orthogonality of the projection residual. In the case of nonparametric instrumental variable regression,
$$
E[v(W_2)\{Y-\gamma_0(W_1)\}]=E(v(W_2)[E(Y \mid W_2)-E\{\gamma_0(W_1) \mid W_2\}])=0.
$$
\end{proof}

\subsection{Taylor expansion}

Train $(\hat{\gamma}_{\ell},\hat{\alpha}_{\ell})$ on observations in $I_{\ell}^c$. Let $n_{\ell}=|I_{\ell}|=n/L$ be the number of observations in $I_{\ell}$. Denote by $E_{\ell}(\cdot)=n_{\ell}^{-1}\sum_{i\in I_{\ell}}(\cdot)$ the average over observations in $I_{\ell}$. Denote by $E_n(\cdot)=n^{-1}\sum_{i=1}^n(\cdot)$ the average over all observations in the sample.

\begin{definition}[Foldwise target and oracle]
\begin{align*}
 \hat{\theta}_{\ell}&=E_{\ell} [m(W,\hat{\gamma}_{\ell})+\hat{\alpha}_{\ell}(W)\{Y-\hat{\gamma}_{\ell}(W)\}];\\
    \bar{\theta}_{\ell}&=E_{\ell} [m(W,\gamma_0)+\alpha_0^{\min}(W)\{Y-\gamma_0(W)\}].
\end{align*}
\end{definition}

\begin{proposition}[Taylor expansion]\label{prop:Taylor}
Let $u=\hat{\gamma}_{\ell}-\gamma_0$ and $v=\hat{\alpha}_{\ell}-\alpha_0^{\min}$. Then $n_{\ell}^{1/2}(\hat{\theta}_{\ell}-\bar{\theta}_{\ell})=\sum_{j=1}^3 \Delta_{j{\ell}}$ where
\begin{align*}
    \Delta_{1{\ell}}&=n_{\ell}^{1/2}E_{\ell}\{m(W,u)-\alpha_0^{\min}(W)u(W)\}; \\
    \Delta_{2{\ell}}&=n_{\ell}^{1/2}E_{\ell}[v(W)\{Y-\gamma_0(W)\}]; \\
    \Delta_{3{\ell}}&=\frac{n_{\ell}^{1/2}}{2}E_{\ell} \{-u(W)v(W)\}.
\end{align*}
\end{proposition}

\begin{proof}
An exact Taylor expansion gives
$$
\psi(w,\theta_0,\hat{\gamma}_{\ell},\hat{\alpha}_{\ell})-\psi_0(w)
=\{\partial_{\gamma} \psi_0(w)\}(u)+\{\partial_{\alpha} \psi_0(w)\}(v)+\frac{1}{2}\{\partial^2_{\gamma,\alpha} \psi_0(w)\}(u,v).
$$
Averaging over observations in $I_{\ell}$
\begin{align*}
    \hat{\theta}_{\ell}-\bar{\theta}_{\ell}
    &=E_{\ell}\{\psi(W,\theta_0,\hat{\gamma}_{\ell},\hat{\alpha}_{\ell})\}-E_{\ell}\{\psi_0(W)\} \\
    &=E_{\ell}\{\partial_{\gamma} \psi_0(W)\}(u)+E_{\ell}\{\partial_{\alpha} \psi_0(W)\}(v)+\frac{1}{2}E_{\ell}\{\partial^2_{\gamma,\alpha} \psi_0(W)\}(u,v).
\end{align*}
Finally appeal to Proposition~\ref{prop:deriv}.
\end{proof}

\subsection{Residuals}

\begin{proposition}[Residuals]\label{prop:resid}
Suppose Assumption~\ref{assumption:cont} holds and
$$
E[\{Y-\gamma_0(W)\}^2 \mid W ]\leq \bar{\sigma}^2,\quad \|\alpha_0^{\min}\|_{\infty}\leq\bar{\alpha}.
$$
Then with probability $1-\epsilon/L$,
\begin{align*}
    |\Delta_{1\ell}|&\leq t_1=\left(\frac{6L}{\epsilon}\right)^{1/2}(\bar{Q}+\bar{\alpha}^2)^{1/2}\{\mathcal{R}(\hat{\gamma}_{\ell})\}^{q/2}; \\
    |\Delta_{2\ell}|&\leq t_2= \left(\frac{3L}{\epsilon}\right)^{1/2}\bar{\sigma}\{\mathcal{R}(\hat{\alpha}_{\ell})\}^{1/2};\\
    |\Delta_{3\ell}|&\leq t_3= \frac{3L^{1/2}}{2\epsilon}\{n\mathcal{R}(\hat{\gamma}_{\ell})\mathcal{R}(\hat{\alpha}_{\ell})\}^{1/2}.
\end{align*}
\end{proposition}

\begin{proof}
We proceed in steps.
\begin{enumerate}
    \item Markov inequality implies
    \begin{align*}
        \text{pr}(|\Delta_{1\ell}|>t_1)&\leq \frac{E(\Delta^2_{1\ell})}{t_1^2};\\
        \text{pr}(|\Delta_{2\ell}|>t_2)&\leq \frac{E(\Delta^2_{2\ell})}{t_2^2}; \\
        \text{pr}(|\Delta_{3\ell}|>t_3)&\leq \frac{E(|\Delta_{3\ell}|)}{t_3}.
    \end{align*}
    \item Law of iterated expectations implies
    \begin{align*}
        E(\Delta^2_{1\ell})&=E\{E(\Delta^2_{1\ell}\mid I^c_{\ell})\};\\
        E(\Delta^2_{2\ell})&=E\{E(\Delta^2_{2\ell}\mid I^c_{\ell})\}; \\
        E(|\Delta_{3\ell}|)&=E\{E(|\Delta_{3\ell}|\mid I^c_{\ell})\}.
    \end{align*}
    \item Bounding conditional moments.
    
    Conditional on $I_{\ell}^c$, $(u,v)$ are nonrandom. Moreover, observations within fold $I_{\ell}$ are independent and identically distributed. Hence by Proposition~\ref{prop:mean_zero} and assumption of finite $(\bar{Q},\bar{\alpha})$
    \begin{align*}
        E(\Delta^2_{1\ell}\mid I^c_{\ell})
        &=E \left([n_{\ell}^{1/2}E_{\ell}\{m(W,u)-\alpha_0^{\min}(W)u(W)\}]^2 \mid I^c_{\ell}\right) \\
        &=E \left[ \frac{n_{\ell}}{n^2_{\ell}} \sum_{i,j\in I_{\ell}} \{m(W_i,u)-\alpha_0^{\min}(W_i)u(W_i)\}\{m(W_j,u)-\alpha_0^{\min}(W_j)u(W_j)\} \mid I^c_{\ell}\right] \\
        &= \frac{n_{\ell}}{n^2_{\ell}} \sum_{i,j\in I_{\ell}}E \left[ \{m(W_i,u)-\alpha_0^{\min}(W_i)u(W_i)\}\{m(W_j,u)-\alpha_0^{\min}(W_j)u(W_j)\} \mid I^c_{\ell}\right] \\
        &= \frac{n_{\ell}}{n^2_{\ell}} \sum_{i\in I_{\ell}}E \left[ \{m(W_i,u)-\alpha_0^{\min}(W_i)u(W_i)\}^2 \mid I^c_{\ell}\right] \\
        &=E[\{m(W,u)-\alpha_0^{\min}(W)u(W)\}^2\mid I^c_{\ell}] \\
        &\leq 2 E\{m(W,u)^2\mid I^c_{\ell}\}+2E[\{\alpha_0^{\min}(W)u(W)\}^2\mid I^c_{\ell}] \\
        &\leq 2(\bar{Q}+\bar{\alpha}^2)\mathcal{R}(\hat{\gamma}_{\ell})^q.
    \end{align*}
    Similarly by Proposition~\ref{prop:mean_zero} and assumption of finite $\bar{\sigma}$
    \begin{align*}
        E(\Delta^2_{2\ell}\mid I^c_{\ell})
       &=E \left\{(n_{\ell}^{1/2}E_{\ell}[v(W)\{Y-\gamma_0(W)\}])^2 \mid I^c_{\ell}\right\} \\
        &=E \left[ \frac{n_{\ell}}{n^2_{\ell}} \sum_{i,j\in I_{\ell}} v(W_i)\{Y_i-\gamma_0(W_i)\} v(W_j)\{Y_j-\gamma_0(W_j)\} \mid I^c_{\ell}\right] \\
        &= \frac{n_{\ell}}{n^2_{\ell}} \sum_{i,j\in I_{\ell}}E \left[ v(W_i)\{Y_i-\gamma_0(W_i)\} v(W_j)\{Y_j-\gamma_0(W_j)\} \mid I^c_{\ell}\right] \\
        &= \frac{n_{\ell}}{n^2_{\ell}} \sum_{i\in I_{\ell}}E \left[ v(W_i)^2\{Y_i-\gamma_0(W_i)\}^2 \mid I^c_{\ell}\right] \\
        &=E[v(W)^2\{Y-\gamma_0(W)\}^2\mid I^c_{\ell}] \\
        &=E\left(v(W)^2 E[\{Y-\gamma_0(W)\}^2\mid W,I^c_{\ell}]\mid I^c_{\ell}\right) \\
        &\leq \bar{\sigma}^2 \mathcal{R}(\hat{\alpha}_{\ell}).
    \end{align*}
    Finally by Cauchy Schwarz inequality
    \begin{align*}
        E(|\Delta_{3\ell}|\mid I^c_{\ell})
        &=\frac{n_{\ell}^{1/2}}{2}  E\{|-u(W)v(W)|\mid I^c_{\ell}\} \\
        &\leq \frac{n_{\ell}^{1/2}}{2} [E\{u(W)^2\mid I^c_{\ell}\}]^{1/2} [E\{v(W)^2\mid I^c_{\ell}\}]^{1/2} \\
        &=\frac{n_{\ell}^{1/2}}{2}  \{\mathcal{R}(\hat{\gamma}_{\ell})\}^{1/2}\{\mathcal{R}(\hat{\alpha}_{\ell})\}^{1/2}.
    \end{align*}
    \item Collecting results gives
     \begin{align*}
        \text{pr}(|\Delta_{1\ell}|>t_1)&\leq \frac{2 (\bar{Q}+\bar{\alpha}^2) \mathcal{R}(\hat{\gamma}_{\ell})^q}{t_1^2}=\frac{\epsilon}{3L};\\
        \text{pr}(|\Delta_{2\ell}|>t_2)&\leq \frac{\bar{\sigma}^2 \mathcal{R}(\hat{\alpha}_{\ell})}{t_2^2}=\frac{\epsilon}{3L}; \\
        \text{pr}(|\Delta_{3\ell}|>t_3)&\leq \frac{n_{\ell}^{1/2} \{\mathcal{R}(\hat{\gamma}_{\ell})\}^{1/2}\{\mathcal{R}(\hat{\alpha}_{\ell})\}^{1/2}}{2t_3}=\frac{\epsilon}{3L}.
    \end{align*}
    Therefore with probability $1-\epsilon/L$, the following inequalities hold:
\begin{align*}
    |\Delta_{1\ell}|&\leq t_1=\left(\frac{6L}{\epsilon}\right)^{1/2}(\bar{Q}+\bar{\alpha}^2)^{1/2}\{\mathcal{R}(\hat{\gamma}_{\ell})\}^{q/2};\\
    |\Delta_{2\ell}|&\leq t_2=\left(\frac{3L}{\epsilon}\right)^{1/2}\bar{\sigma}\{\mathcal{R}(\hat{\alpha}_{\ell})\}^{1/2} ;\\
    |\Delta_{3\ell}|&\leq t_3=\frac{3L}{2\epsilon}n_{\ell}^{1/2}\{\mathcal{R}(\hat{\gamma}_{\ell})\}^{1/2}\{\mathcal{R}(\hat{\alpha}_{\ell})\}^{1/2}.
\end{align*}
Finally recall $n_{\ell}=n/L$.
\end{enumerate}
\end{proof}

\begin{proposition}[Residuals: Alternative path]\label{prop:resid_alt}
Suppose Assumption~\ref{assumption:cont} holds and
$$
E[\{Y-\gamma_0(W)\}^2 \mid W ]\leq \bar{\sigma}^2,\quad \|\alpha_0^{\min}\|_{\infty}\leq\bar{\alpha},\quad \|\hat{\alpha}_{\ell}\|_{\infty}\leq\bar{\alpha}'.
$$
Then with probability $1-\epsilon/L$,
\begin{align*}
    |\Delta_{1\ell}|&\leq t_1=\left(\frac{6L}{\epsilon}\right)^{1/2}(\bar{Q}+\bar{\alpha}^2)^{1/2}\{\mathcal{R}(\hat{\gamma}_{\ell})\}^{q/2}; \\
    |\Delta_{2\ell}|&\leq t_2= \left(\frac{3L}{\epsilon}\right)^{1/2}\bar{\sigma}\{\mathcal{R}(\hat{\alpha}_{\ell})\}^{1/2};\\
    |\Delta_{3\ell}|&\leq t_3= \left(\frac{3L}{4\epsilon}\right)^{1/2}(\bar{\alpha}+\bar{\alpha}')\{\mathcal{R}(\hat{\gamma}_{\ell})\}^{1/2}\\
    &\quad \quad \quad +(4L)^{-1/2} [\{n\mathcal{P}(\hat{\gamma}_{\ell})\mathcal{R}(\hat{\alpha}_{\ell})\}^{1/2} \wedge \{n\mathcal{R}(\hat{\gamma}_{\ell})\mathcal{P}(\hat{\alpha}_{\ell})\}^{1/2}].
\end{align*}
\end{proposition}

\begin{proof}
See Proposition~\ref{prop:resid} for $(t_1,t_2)$. We focus on the alternative bound $t_3$.
\begin{enumerate}
\item Decomposition.

    Write
    $$
    2\Delta_{3\ell}=n_{\ell}^{1/2}E_{\ell} \{-u(W)v(W)\}=\Delta_{3'\ell}+\Delta_{3''\ell},
    $$
    where
    \begin{align*}
        \Delta_{3'\ell}&=n_{\ell}^{1/2}E_{\ell} [-u(W)v(W)+E\{u(W)v(W)\mid I^c_{\ell}\}]; \\
        \Delta_{3''\ell}&= n_{\ell}^{1/2} E\{-u(W)v(W)\mid I^c_{\ell}\}.
    \end{align*}

\item Former term.

By Markov inequality
$$
\text{pr}(|\Delta_{3'\ell}|>t)\leq \frac{E(\Delta^2_{3'\ell})}{t^2}.
$$
Law of iterated expectations implies
$$
E(\Delta^2_{3'\ell})=E\{E(\Delta^2_{3'\ell}\mid I^c_{\ell})\}.
$$
We bound the conditional moment. Conditional on $I_{\ell}^c$, $(u,v)$ are nonrandom. Moreover, observations within fold $I_{\ell}$ are independent and identically distributed. Since each summand in $\Delta_{3'\ell}$ has conditional mean zero by construction, and since $(\bar{\alpha},\bar{\alpha}')$ are finite by hypothesis,
\begin{align*}
        &E(\Delta^2_{3'\ell}\mid I^c_{\ell}) \\
        &=  E\left\{\left(n_{\ell}^{1/2}E_{\ell} [-u(W)v(W)+E\{u(W)v(W)\mid I^c_{\ell}\}]\right)^2 \mid I^c_{\ell}\right\} \\
        &= E\left( \frac{n_{\ell}}{n^2_{\ell}}\sum_{i,j \in I_{\ell}} [-u(W_i)v(W_i)+E\{u(W_i)v(W_i)\mid I^c_{\ell}\}][-u(W_j)v(W_j)+E\{u(W_j)v(W_j)\mid I^c_{\ell}\}] \mid I^c_{\ell}\right) \\
        &= \frac{n_{\ell}}{n^2_{\ell}}\sum_{i,j \in I_{\ell}} E\left(  [-u(W_i)v(W_i)+E\{u(W_i)v(W_i)\mid I^c_{\ell}\}][-u(W_j)v(W_j)+E\{u(W_j)v(W_j)\mid I^c_{\ell}\}] \mid I^c_{\ell}\right) \\
        &= \frac{n_{\ell}}{n^2_{\ell}}\sum_{i \in I_{\ell}} E\left(  [-u(W_i)v(W_i)+E\{u(W_i)v(W_i)\mid I^c_{\ell}\}]^2 \mid I^c_{\ell}\right) \\
        &=E([u(W)v(W)-E\{u(W)v(W)\mid I^c_{\ell}\} ]^2\mid I^c_{\ell}) \\
        &\leq E \{ u(W)^2v(W)^2\mid I^c_{\ell}\} \\
        &\leq (\bar{\alpha}+\bar{\alpha}')^2\mathcal{R}(\hat{\gamma}_{\ell}).
    \end{align*}
    Collecting results gives
    $$
     \text{pr}(|\Delta_{3'\ell}|>t)\leq \frac{(\bar{\alpha}+\bar{\alpha}')^2\mathcal{R}(\hat{\gamma}_{\ell})}{t^2}=\frac{\epsilon}{3L}.
    $$
    Therefore with probability $1-\epsilon/(3L)$,
    $$
    |\Delta_{3'\ell}|\leq t=\left(\frac{3L}{\epsilon}\right)^{1/2}(\bar{\alpha}+\bar{\alpha}')\{\mathcal{R}(\hat{\gamma}_{\ell})\}^{1/2}.
    $$

\item Latter term.

Specializing to nonparametric instrumental variable regression,
\begin{align*}
    E\{-u(W)v(W)\mid I^c_{\ell}\}
    &=E[  E\{-u(W_1)\mid W_2, I^c_{\ell}\}v(W_2) \mid I^c_{\ell}] \\
    &\leq  \{E( [E\{u(W_1)\mid W_2, I^c_{\ell}\}]^2 \mid I^c_{\ell})\}^{1/2}[E\{v(W_2)^2\mid I^c_{\ell}\}]^{1/2} \\
    &=\{\mathcal{P}(\hat{\gamma}_{\ell})\}^{1/2}\{\mathcal{R}(\hat{\alpha}_{\ell})\}^{1/2}.
\end{align*}
Hence
$$
\Delta_{3''\ell} \leq n_{\ell}^{1/2}\{\mathcal{P}(\hat{\gamma}_{\ell})\}^{1/2}\{\mathcal{R}(\hat{\alpha}_{\ell})\}^{1/2}=L^{-1/2}\{n\mathcal{P}(\hat{\gamma}_{\ell})\mathcal{R}(\hat{\alpha}_{\ell})\}^{1/2}.
$$
Likewise
\begin{align*}
    E\{-u(W)v(W)\mid I^c_{\ell}\} 
   &=E[  -u(W_1)E\{v(W_2)\mid W_1, I^c_{\ell}\} \mid I^c_{\ell}] \\
    &\leq  [E\{u(W_1)^2\mid I^c_{\ell}\}]^{1/2}\{E( [E\{v(W_2)\mid W_1, I^c_{\ell}\}]^2 \mid I^c_{\ell})\}^{1/2} \\
    &=\{\mathcal{R}(\hat{\gamma}_{\ell})\}^{1/2}\{\mathcal{P}(\hat{\alpha}_{\ell})\}^{1/2}.
\end{align*}
Hence
$$
\Delta_{3''\ell} \leq n_{\ell}^{1/2}\{\mathcal{R}(\hat{\gamma}_{\ell})\}^{1/2}\{\mathcal{P}(\hat{\alpha}_{\ell})\}^{1/2}=L^{-1/2}\{n\mathcal{R}(\hat{\gamma}_{\ell})\mathcal{P}(\hat{\alpha}_{\ell})\}^{1/2}.
$$

    \item Combining terms.
    
 With probability $1-\epsilon/(3L)$,
\begin{align*}
    |\Delta_3| &\leq t_3= \left(\frac{3L}{4\epsilon}\right)^{1/2}(\bar{\alpha}+\bar{\alpha}')\{\mathcal{R}(\hat{\gamma}_{\ell})\}^{1/2}\\
    &\quad \quad \quad +(4L)^{-1/2} [\{n\mathcal{P}(\hat{\gamma}_{\ell})\mathcal{R}(\hat{\alpha}_{\ell})\}^{1/2} \wedge \{n\mathcal{R}(\hat{\gamma}_{\ell})\mathcal{P}(\hat{\alpha}_{\ell})\}^{1/2}].
\end{align*}
\end{enumerate}
\end{proof}

\subsection{Main argument}

\begin{definition}[Overall target and oracle]
$$
    \hat{\theta}=\frac{1}{L}\sum_{\ell=1}^L \hat{\theta}_{\ell},\quad 
    \bar{\theta}=\frac{1}{L}\sum_{\ell=1}^L \bar{\theta}_{\ell} .
$$
\end{definition}

\begin{proposition}[Oracle approximation]\label{prop:Delta}
Suppose the conditions of Proposition~\ref{prop:resid} hold. Then with probability $1-\epsilon$
$$
\frac{n^{1/2}}{\sigma}|\hat{\theta}-\bar{\theta}|\leq \Delta= \frac{3L}{\epsilon  \sigma}\left[(\bar{Q}^{1/2}+\bar{\alpha})\{\mathcal{R}(\hat{\gamma}_{\ell})\}^{q/2}+\bar{\sigma}\{\mathcal{R}(\hat{\alpha}_{\ell})\}^{1/2}+ \{n\mathcal{R}(\hat{\gamma}_{\ell})\mathcal{R}(\hat{\alpha}_{\ell})\}^{1/2}\right].
$$
\end{proposition}

\begin{proof}
We proceed in steps.
\begin{enumerate}
    \item Decomposition.
    
    By Proposition~\ref{prop:Taylor}, write
    \begin{align*}
    n^{1/2}(\hat{\theta}-\bar{\theta})
    &=\frac{n^{1/2}}{n_{\ell}^{1/2}}\frac{1}{L} \sum_{\ell=1}^L n_{\ell}^{1/2} (\hat{\theta}_{\ell}-\bar{\theta}_{\ell}) \\
    &= L^{1/2}\frac{1}{L} \sum_{\ell=1}^L \sum_{j=1}^3 \Delta_{j\ell}.
\end{align*} 

    \item Union bound.
    
    Define the events
$$
\mathcal{E}_{\ell}=\{\text{for all } j \; (j=1,2,3),\; |\Delta_{j\ell}|\leq t_j\},\quad \mathcal{E}=\cap_{\ell=1}^L \mathcal{E}_{\ell},\quad \mathcal{E}^c=\cup_{\ell=1}^L \mathcal{E}^c_{\ell}.
$$
Hence by the union bound and Proposition~\ref{prop:resid},
$$
\text{pr}(\mathcal{E}^c)\leq \sum_{\ell=1}^L \text{pr}(\mathcal{E}^c_{\ell}) \leq L\frac{\epsilon}{L}=\epsilon.
$$
    \item Collecting results.

    Therefore with probability $1-\epsilon$,
\begin{align*}
 n^{1/2}|\hat{\theta}-\bar{\theta}|&\leq L^{1/2}\frac{1}{L} \sum_{\ell=1}^L \sum_{j=1}^3 |\Delta_{jk}| \\
 &\leq L^{1/2}\frac{1}{L} \sum_{\ell=1}^L \sum_{j=1}^3 t_j \\
 &=L^{1/2}\sum_{j=1}^3 t_j.
\end{align*}
Finally, we simplify $(t_j)$. For $a,b>0$, $(a+b)^{1/2}\leq a^{1/2}+b^{1/2}$. Moreover, $3>6^{1/2}>3^{1/2}>3/2$. Finally, for $\epsilon\leq 1$, $\epsilon^{-1/2}\leq \epsilon^{-1}$. In summary
\begin{align*}
    t_1&=\left(\frac{6L}{\epsilon}\right)^{1/2}(\bar{Q}+\bar{\alpha}^2)^{1/2}\{\mathcal{R}(\hat{\gamma}_{\ell})\}^{q/2}
   \leq \frac{3L^{1/2}}{\epsilon}(\bar{Q}^{1/2}+\bar{\alpha})\{\mathcal{R}(\hat{\gamma}_{\ell})\}^{q/2}; \\
    t_2&=\left(\frac{3L}{\epsilon}\right)^{1/2}\bar{\sigma}\{\mathcal{R}(\hat{\alpha}_{\ell})\}^{1/2} 
   \leq \frac{3L^{1/2}}{\epsilon}\bar{\sigma}\{\mathcal{R}(\hat{\alpha}_{\ell})\}^{1/2}; \\
    t_3&=\frac{3L^{1/2}}{2\epsilon}\{n\mathcal{R}(\hat{\gamma}_{\ell})\mathcal{R}(\hat{\alpha}_{\ell})\}^{1/2}
    \leq \frac{3L^{1/2}}{\epsilon}\{n\mathcal{R}(\hat{\gamma}_{\ell})\mathcal{R}(\hat{\alpha}_{\ell})\}^{1/2}.
\end{align*}
\end{enumerate}
\end{proof}

\begin{proposition}[Oracle approximation: Alternative path]\label{prop:Delta_alt}
Suppose the conditions of Proposition~\ref{prop:resid_alt} hold. Then with probability $1-\epsilon$
\begin{align*}
    \frac{n^{1/2}}{\sigma}|\hat{\theta}-\bar{\theta}|\leq 
\Delta&= \frac{4 L}{\epsilon^{1/2}  \sigma}\left[(\bar{Q}^{1/2}+\bar{\alpha}+\bar{\alpha}')\{\mathcal{R}(\hat{\gamma}_{\ell})\}^{q/2}+\bar{\sigma}\{\mathcal{R}(\hat{\alpha}_{\ell})\}^{1/2}\right]\\
&\quad +\frac{1}{2L^{1/2}\sigma}[\{n\mathcal{P}(\hat{\gamma}_{\ell})\mathcal{R}(\hat{\alpha}_{\ell})\}^{1/2} \wedge \{n\mathcal{R}(\hat{\gamma}_{\ell})\mathcal{P}(\hat{\alpha}_{\ell})\}^{1/2}].
\end{align*}
\end{proposition}

\begin{proof}
As in Proposition~\ref{prop:Delta}, Propositions~\ref{prop:Taylor} and~\ref{prop:resid_alt} imply that with probability $1-\epsilon$
\begin{align*}
 n^{1/2}|\hat{\theta}-\bar{\theta}|&\leq L^{1/2}\sum_{j=1}^3 t_j.
\end{align*}
Finally, we simplify $(t_j)$. For $a,b>0$, $(a+b)^{1/2}\leq a^{1/2}+b^{1/2}$. In summary,
\begin{align*}
    t_1&=\left(\frac{6L}{\epsilon}\right)^{1/2}(\bar{Q}+\bar{\alpha}^2)^{1/2}\{\mathcal{R}(\hat{\gamma}_{\ell})\}^{q/2}
    \leq \left(\frac{6L}{\epsilon}\right)^{1/2}(\bar{Q}^{1/2}+\bar{\alpha})\{\mathcal{R}(\hat{\gamma}_{\ell})\}^{q/2}; \\
    t_2&= \left(\frac{3L}{\epsilon}\right)^{1/2}\bar{\sigma}\{\mathcal{R}(\hat{\alpha}_{\ell})\}^{1/2};\\
    t_3&= \left(\frac{3L}{4\epsilon}\right)^{1/2}(\bar{\alpha}+\bar{\alpha}')\{\mathcal{R}(\hat{\gamma}_{\ell})\}^{1/2}\\
    &\quad  +(4L)^{-1/2} [\{n\mathcal{P}(\hat{\gamma}_{\ell})\mathcal{R}(\hat{\alpha}_{\ell})\}^{1/2} \wedge \{n\mathcal{R}(\hat{\gamma}_{\ell})\mathcal{P}(\hat{\alpha}_{\ell})\}^{1/2}].
\end{align*}
Finally note $6^{1/2}+(3/4)^{1/2}\leq 4$ when combining terms from $t_1$ and $t_3$.
\end{proof}

\begin{lemma}[Berry Esseen Theorem \cite{shevtsova2011absolute}]\label{lem:berry}
Suppose $(Z_i)$ $(i=1,...,n)$ are independent and identically distributed random variables with $E(Z_i)=0$, $E(Z_i^2)=\sigma^2$, and $E(|Z_i|^3)=\xi^3$. Then
$$
\sup_{z \in\mathbb{R}}\left|\text{\normalfont pr}\left\{\frac{n^{1/2}}{\sigma}E_n(Z_i)\leq z\right\}-\Phi(z)\right|\leq c^{BE} \left(\frac{\xi}{\sigma}\right)^3n^{-\frac{1}{2}},
$$
where $c^{BE}=0.4748$ and $\Phi(z)$ is the standard Gaussian cumulative distribution function.
\end{lemma}

\begin{proof}[of Theorem~\ref{thm:dml}]
Fix $z$ in $\mathbb{R}$. First, we show that
$$
\text{\normalfont pr}\left\{\frac{n^{1/2}}{\sigma}(\hat{\theta}-\theta_0)\leq z\right\}-\Phi(z)\leq c^{BE} \left(\frac{\xi}{\sigma}\right)^3n^{-\frac{1}{2}} +\frac{\Delta}{(2\pi)^{1/2}}+\epsilon,
$$
where $\Delta$ is defined in Propositions~\ref{prop:Delta} and~\ref{prop:Delta_alt}. We proceed in steps.
\begin{enumerate}
    \item High probability bound. 
    
    By Propositions~\ref{prop:Delta} and~\ref{prop:Delta_alt}, with probability $1-\epsilon$, 
$$
\frac{n^{1/2}}{\sigma}(\bar{\theta}-\hat{\theta}) \leq \frac{n^{1/2}}{\sigma}|\hat{\theta}-\bar{\theta}|\leq \Delta.
$$
Observe that
\begin{align*}
    \text{\normalfont pr}\left\{\frac{n^{1/2}}{\sigma}(\hat{\theta}-\theta_0)\leq z\right\}
    &=\text{\normalfont pr}\left\{\frac{n^{1/2}}{\sigma}(\bar{\theta}-\theta_0)\leq z+\frac{n^{1/2}}{\sigma}(\bar{\theta}-\hat{\theta})\right\}\\
    &\leq \text{\normalfont pr}\left\{\frac{n^{1/2}}{\sigma}(\bar{\theta}-\theta_0)\leq z+\Delta\right\}+\epsilon.
\end{align*}
    \item Mean value theorem.
    
    Let $\phi(z)$ be the standard Gaussian probability density function. There exists some $z'$ such that
    $$
\Phi(z+\Delta)-\Phi(z) =\phi(z') \Delta \leq \frac{\Delta}{\sqrt{2 \pi}}.
$$
    \item Berry Esseen theorem.
    
    Observe that
$$
   \bar{\theta}-\theta_0
   =E_n[m(W,\gamma_0)+\alpha_0^{\min}(W)\{Y-\gamma_0(W)\}]-\theta_0 
   =E_n[\psi_0(W)].
$$
    Therefore taking $Z_i=\psi_0(W_i)$ in Lemma~\ref{lem:berry},
$$
\sup_{z''}\left|\text{\normalfont pr}\left\{\frac{n^{1/2}}{\sigma}(\bar{\theta}-\theta_0)\leq z''\right\}-\Phi(z'')\right|\leq c^{BE} \left(\frac{\xi}{\sigma}\right)^3n^{-\frac{1}{2}}.
$$
Hence by the high probability bound and mean value theorem steps above, taking $z''=z+\Delta$
\begin{align*}
    &\text{\normalfont pr}\left\{\frac{n^{1/2}}{\sigma}(\hat{\theta}-\theta_0)\leq z\right\}-\Phi(z)\\
    &\leq \text{\normalfont pr}\left\{\frac{n^{1/2}}{\sigma}(\bar{\theta}-\theta_0)\leq z+\Delta\right\}-\Phi(z)+\epsilon \\
    &=\text{\normalfont pr}\left\{\frac{n^{1/2}}{\sigma}(\bar{\theta}-\theta_0)\leq z+\Delta\right\}-\Phi(z+\Delta)+\Phi(z+\Delta)-\Phi(z)+\epsilon \\
    &\leq c^{BE} \left(\frac{\xi}{\sigma}\right)^3n^{-\frac{1}{2}}+\frac{\Delta}{\sqrt{2 \pi}}+\epsilon.
\end{align*}
\end{enumerate}
Next, we show that
$$
\Phi(z)-\text{\normalfont pr}\left\{\frac{n^{1/2}}{\sigma}(\hat{\theta}-\theta_0)\leq z\right\}\leq c^{BE}\left(\frac{\xi}{\sigma}\right)^3n^{-\frac{1}{2}} +\frac{\Delta}{(2\pi)^{1/2}}+\epsilon.
$$
where $\Delta$ is defined in Propositions~\ref{prop:Delta} and~\ref{prop:Delta_alt}. We proceed in steps.
\begin{enumerate}
    \item High probability bound. 
    
    By Propositions~\ref{prop:Delta} and~\ref{prop:Delta_alt}, with probability $1-\epsilon$, 
$$
\frac{n^{1/2}}{\sigma}(\hat{\theta}-\bar{\theta}) \leq \frac{n^{1/2}}{\sigma}|\hat{\theta}-\bar{\theta}|\leq \Delta,
$$
hence
$$
z-\Delta \leq z-\frac{n^{1/2}}{\sigma}(\hat{\theta}-\bar{\theta}).
$$
Observe that
\begin{align*}
    \text{\normalfont pr}\left\{\frac{n^{1/2}}{\sigma}(\bar{\theta}-\theta_0)\leq z-\Delta \right\}
    &\leq \text{\normalfont pr}\left\{\frac{n^{1/2}}{\sigma}(\bar{\theta}-\theta_0)\leq z-\frac{n^{1/2}}{\sigma}(\hat{\theta}-\bar{\theta}) \right\}+\epsilon \\
    &=\text{\normalfont pr}\left\{\frac{n^{1/2}}{\sigma}(\hat{\theta}-\theta_0)\leq z \right\}+\epsilon.
\end{align*}
    \item Mean value theorem. 
    
    There exists some $z'$ such that
    $$
\Phi(z)-\Phi(z-\Delta)=\phi(z') \Delta \leq  \frac{\Delta}{\sqrt{2 \pi}}.
$$
    \item Berry Esseen theorem.
    
    As argued above,
$$
\sup_{z''}\left|\text{\normalfont pr}\left\{\frac{n^{1/2}}{\sigma}(\bar{\theta}-\theta_0)\leq z''\right\}-\Phi(z'')\right|\leq c^{BE} \left(\frac{\xi}{\sigma}\right)^3n^{-\frac{1}{2}}.
$$
Hence by the mean value theorem and high probability bound steps above, taking $z''=z-\Delta$
\begin{align*}
    &\Phi(z)-\text{\normalfont pr}\left\{\frac{n^{1/2}}{\sigma}(\hat{\theta}-\theta_0)\leq z\right\}\\
    &\leq \Phi(z)-\text{\normalfont pr}\left\{\frac{n^{1/2}}{\sigma}(\bar{\theta}-\theta_0)\leq z-\Delta\right\}+\epsilon \\
    &=\Phi(z)-\Phi(z-\Delta)+\Phi(z-\Delta)-\text{\normalfont pr}\left\{\frac{n^{1/2}}{\sigma}(\bar{\theta}-\theta_0)\leq z-\Delta\right\}+\epsilon \\
    &\leq \frac{\Delta}{\sqrt{2 \pi}}+c^{BE} \left(\frac{\xi}{\sigma}\right)^3n^{-\frac{1}{2}}+\epsilon.
\end{align*}
\end{enumerate}
\end{proof}

\subsection{Variance estimation}

Recall that $E_{\ell}(\cdot)=n_{\ell}^{-1}\sum_{i\in I_{\ell}}(\cdot)$ means the average over observations in $I_{\ell}$ and $E_n(\cdot)=n^{-1}\sum_{i=1}^n(\cdot)$ means the average over all observations in the sample.

\begin{definition}[Shorter notation]
For $i \text{ in }  I_{\ell}$, define
\begin{align*}
    \psi_0(W_i)&=\psi(W_i,\theta_0,\gamma_0,\alpha_0^{\min}); \\
    \hat{\psi}(W_i)&=\psi(W_i,\hat{\theta},\hat{\gamma}_{\ell},\hat{\alpha}_{\ell}).
\end{align*}
\end{definition}

\begin{proposition}[Foldwise second moment]\label{prop:foldwise2}
$$
E_{\ell}[\{\hat{\psi}(W)-\psi_0(W)\}^2]\leq 4\left\{(\hat{\theta}-\theta_0)^2+\sum_{j=4}^6 \Delta_{j\ell}\right\},
$$
where
\begin{align*}
    \Delta_{4\ell} &=E_{\ell}\{m(W,u)^2\} ;\\
    \Delta_{5\ell}&=E_{\ell}[\{\hat{\alpha}_{\ell}(W)u(W)\}^2]; \\
    \Delta_{6\ell}&=E_{\ell}[v(W)^2\{Y-\gamma_0(W)\}^2].
\end{align*}
\end{proposition}

\begin{proof}
Write
\begin{align*}
    \hat{\psi}(W_i)-\psi_0(W_i)
    &=m(W_i,\hat{\gamma}_{\ell})+\hat{\alpha}_{\ell}(W_i)\{Y_i-\hat{\gamma}_{\ell}(W_i)\}-\hat{\theta}\\
    &\quad -\left[m(W_i,\gamma_0)+\alpha_0^{\min}(W_i)\{Y_i-\gamma_0(W_i)\}-\theta_0\right] \\
    &\quad\pm \hat{\alpha}_{\ell}\{Y-\gamma_0(W_i)\}\\
    &=(\theta_0-\hat{\theta})+m(W_i,u)-\hat{\alpha}_{\ell}(W_i)u(W_i)+v(W_i)\{Y-\gamma_0(W_i)\}.
\end{align*}
Hence
$$
\{\hat{\psi}(W_i)-\psi_0(W_i)\}^2\leq 4\left[ (\theta_0-\hat{\theta})^2+m(W_i,u)^2+\{\hat{\alpha}_{\ell}(W_i)u(W_i)\}^2+v(W_i)^2\{Y-\gamma_0(W_i)\}^2\right].
$$
Finally take $E_{\ell}(\cdot)$ of both sides.
\end{proof}

\begin{proposition}[Residuals]\label{prop:resid2}
Suppose Assumption~\ref{assumption:cont} holds and
$$
E[\{Y-\gamma_0(W)\} \mid W ]^2\leq \bar{\sigma}^2,\quad \|\hat{\alpha}_{\ell}\|_{\infty}\leq\bar{\alpha}'.
$$
Then with probability $1-\epsilon'/(2L)$,
\begin{align*}
    \Delta_{4\ell}&\leq t_4=\frac{6L}{\epsilon'}\bar{Q}\mathcal{R}(\hat{\gamma}_{\ell})^q; \\
     \Delta_{5\ell}&\leq t_5=\frac{6L}{\epsilon'}(\bar{\alpha}')^2\mathcal{R}(\hat{\gamma}_{\ell}); \\
     \Delta_{6\ell}&\leq t_6=\frac{6L}{\epsilon'}\bar{\sigma}^2\mathcal{R}(\hat{\alpha}_{\ell}). 
\end{align*}
\end{proposition}

\begin{proof}
We proceed in steps analogous to Proposition~\ref{prop:resid}.
\begin{enumerate}
    \item Markov inequality implies
     \begin{align*}
        \text{pr}(|\Delta_{4\ell}|>t_4)&\leq \frac{E(|\Delta_{4\ell}|)}{t_4};\\
        \text{pr}(|\Delta_{5\ell}|>t_5)&\leq \frac{E(|\Delta_{5\ell}|)}{t_5}; \\
        \text{pr}(|\Delta_{6\ell}|>t_6)&\leq \frac{E(|\Delta_{6\ell}|)}{t_6}.
    \end{align*}
    \item Law of iterated expectations implies
    \begin{align*}
        E(|\Delta_{4\ell}|)&=E\{E(|\Delta_{4\ell}|\mid I^c_{\ell})\};\\
        E(|\Delta_{5\ell}|)&=E\{E(|\Delta_{5\ell}|\mid I^c_{\ell})\}; \\
        E(|\Delta_{6\ell}|)&=E\{E(|\Delta_{6\ell}|\mid I^c_{\ell})\}.
    \end{align*}
    \item Bounding conditional moments.
    
     Conditional on $I_{\ell}^c$, $(u,v)$ are nonrandom. Moreover, observations within fold $I_{\ell}$ are independent and identically distributed. Hence by assumption of finite $(\bar{Q},\bar{\sigma},\bar{\alpha}')$
    $$
        E(|\Delta_{4\ell}|\mid I^c_{\ell})= E(\Delta_{4\ell}\mid I^c_{\ell})
        =E[\{m(W,u)\}^2\mid I^c_{\ell}]
        \leq  \bar{Q} \mathcal{R}(\hat{\gamma}_{\ell})^q.
   $$
   Similarly
    $$
       E(|\Delta_{5\ell}|\mid I^c_{\ell})= E(\Delta_{5\ell}\mid I^c_{\ell})
        =E[\{\hat{\alpha}_{\ell}(W)u(W)\}^2\mid I^c_{\ell}]
        \leq  (\bar{\alpha}')^2 \mathcal{R}(\hat{\gamma}_{\ell}).
   $$
    Finally 
    \begin{align*}
       E(|\Delta_{6\ell}|\mid I^c_{\ell})&=
       E(\Delta_{6\ell}\mid I^c_{\ell}) \\
        &=E[v(W)^2\{Y-\gamma_0(W)\}^2\mid I^c_{\ell}] \\
        &=E\{v(W)^2 E[\{Y-\gamma_0(W)\}^2\mid W,I^c_{\ell}]\mid I^c_{\ell}\} \\
        &\leq \bar{\sigma}^2 \mathcal{R}(\hat{\alpha}_{\ell}).
    \end{align*}
    \item Collecting results gives
     \begin{align*}
        \text{pr}(|\Delta_{4\ell}|>t_4)&\leq \frac{\bar{Q} \mathcal{R}(\hat{\gamma}_{\ell})^q}{t_4}=\frac{\epsilon'}{6L}\\
        \text{pr}(|\Delta_{5\ell}|>t_5)&\leq \frac{(\bar{\alpha}')^2 \mathcal{R}(\hat{\gamma}_{\ell})}{t_5}=\frac{\epsilon'}{6L} \\
        \text{pr}(|\Delta_{6\ell}|>t_6)&\leq \frac{\bar{\sigma}^2 \mathcal{R}(\hat{\alpha}_{\ell})}{t_6}=\frac{\epsilon'}{6L}
    \end{align*}
    Therefore with probability $1-\epsilon'/(2L)$, the following inequalities hold:
\begin{align*}
    |\Delta_{4\ell}|&\leq t_4=\frac{6L}{\epsilon'}\bar{Q} \mathcal{R}(\hat{\gamma}_{\ell})^q;\\
    |\Delta_{5\ell}|&\leq t_5=\frac{6L}{\epsilon'}(\bar{\alpha}')^2 \mathcal{R}(\hat{\gamma}_{\ell});\\
    |\Delta_{6\ell}|&\leq t_6=\frac{6L}{\epsilon'}\bar{\sigma}^2 \mathcal{R}(\hat{\alpha}_{\ell}).
\end{align*}
\end{enumerate}
\end{proof}

\begin{proposition}[Oracle approximation]\label{prop:Delta2}
Suppose the conditions of Proposition~\ref{prop:resid2} hold. Then with probability $1-\epsilon'/2$
$$
E_n[\{\hat{\psi}(W)-\psi_0(W)\}^2]\leq \Delta'=4(\hat{\theta}-\theta_0)^2+\frac{24 L}{\epsilon'}\left[\{\bar{Q}+(\bar{\alpha}')^2\}\mathcal{R}(\hat{\gamma}_{\ell})^q+\bar{\sigma}^2\mathcal{R}(\hat{\alpha}_{\ell})\right].
$$
\end{proposition}

\begin{proof}
We proceed in steps analogous to Proposition~\ref{prop:Delta}.
\begin{enumerate}
    \item Decomposition.
    
    By Proposition~\ref{prop:foldwise2}
    \begin{align*}
    E_n[\{\hat{\psi}(W)-\psi_0(W)\}^2]
    &=\frac{1}{L} \sum_{\ell=1}^L E_{\ell}[\{\hat{\psi}(W)-\psi_0(W)\}^2] \\
    &\leq 4(\hat{\theta}-\theta_0)^2+\frac{4}{L} \sum_{\ell=1}^L \sum_{j=4}^6 \Delta_{j\ell}.
\end{align*} 

    \item Union bound.
    
    Define the events
$$
\mathcal{E}'_{\ell}=\{\text{for all } j \; (j=4,5,6),\; |\Delta_{j\ell}|\leq t_j\},\quad \mathcal{E}'=\cap_{\ell=1}^L \mathcal{E}'_{\ell},\quad (\mathcal{E}')^c=\cup_{\ell=1}^L (\mathcal{E}_{\ell}')^c.
$$
Hence by the union bound and Proposition~\ref{prop:resid2},
$$
\text{pr}\{(\mathcal{E}')^c\}\leq \sum_{\ell=1}^L \text{pr}\{(\mathcal{E}_{\ell}')^c\} \leq L\frac{\epsilon'}{2L}=\frac{\epsilon'}{2}.
$$
    \item Collecting results.

    Therefore with probability $1-\epsilon'/2$,
\begin{align*}
 E_n[\{\hat{\psi}(W)-\psi_0(W)\}^2]&\leq 4(\hat{\theta}-\theta_0)^2+\frac{4}{L} \sum_{\ell=1}^L \sum_{j=4}^6 |\Delta_{j\ell}| \\
 &\leq 4(\hat{\theta}-\theta_0)^2+\frac{4}{L} \sum_{\ell=1}^L \sum_{j=4}^6 t_j \\
 &=4(\hat{\theta}-\theta_0)^2+4 \sum_{j=4}^6 t_j.
\end{align*}
\end{enumerate}
\end{proof}

\begin{proposition}[Markov inequality]\label{prop:other_half}
Recall $\sigma^2=E\{\psi_0(W)^2\}$ and $\zeta^4=E\{\psi_0(W)^4\}$. Suppose $\zeta<\infty$. Then with probability $1-\epsilon'/2$
$$
|E_n\{\psi_0(W)^2\}-\sigma^2|\leq \Delta''=\left(\frac{2}{\epsilon'}\right)^{1/2}\frac{\zeta^2}{n^{1/2}}.
$$
\end{proposition}

\begin{proof}
Let
$$
A=\psi_0(W)^2,\quad \bar{A}=E_n(A).
$$
Observe that
$$
E(\bar{A})=E(A)=E\{\psi_0(W)^2\}=\sigma^2,\quad 
\text{var}(\bar{A})=\frac{\text{var}(A)}{n}
\leq \frac{E(A^2)}{n}
=\frac{E\{\psi_0(W)^4\}}{n}=\frac{\zeta^4}{n}.
$$
By Markov inequality
$$
\text{pr}[|E_n\{\psi_0(W)^2\}-\sigma^2|>t]=\text{pr}\{|\bar{A}-E(\bar{A})|>t\}\leq \frac{\text{var}(\bar{A})}{t^2}\leq \frac{\zeta^4}{nt^2}.
$$
Solving,
$$
\frac{\zeta^4}{nt^2}=\frac{\epsilon'}{2} \iff t=\left(\frac{2}{\epsilon'}\right)^{1/2}\frac{\zeta^2}{n^{1/2}}
.$$
\end{proof}

\begin{proof}[of Theorem~\ref{thm:var}]
We proceed in steps.
\begin{enumerate}
    \item Decomposition of variance estimator.
    
    Write
    \begin{align*}
        \hat{\sigma}^2
        &=E_n\{\hat{\psi}(W)^2\}\\
        &=E_n[\{\hat{\psi}(W)-\psi_0(W)+\psi_0(W)\}^2]\\
        &=E_n[\{\hat{\psi}(W)-\psi_0(W)\}^2]+2E_n[\{\hat{\psi}(W)-\psi_0(W)\}\psi_0(W)]+E_n\{\psi_0(W)^2\}.
    \end{align*}
Hence
$$
\hat{\sigma}^2-E_n\{\psi_0(W)^2\}=E_n[\{\hat{\psi}(W)-\psi_0(W)\}^2]+2E_n[\{\hat{\psi}(W)-\psi_0(W)\}\psi_0(W)].
$$
    
    \item Decomposition of difference.
    
    Next write
    $$
\hat{\sigma}^2-\sigma^2=[\hat{\sigma}^2-E_n\{\psi_0(W)^2\}]+[E_n\{\psi_0(W)^2\}-\sigma^2].
$$
Focusing on the former term
$$
\hat{\sigma}^2-E_n\{\psi_0(W)^2\}=E_n[\{\hat{\psi}(W)-\psi_0(W)\}^2]+2E_n[\{\hat{\psi}(W)-\psi_0(W)\}\psi_0(W)].
$$
Moreover
\begin{align*}
    E_n[\{\hat{\psi}(W)-\psi_0(W)\}\psi_0(W)] &\leq \left(E_n[\{\hat{\psi}(W)-\psi_0(W)\}^2]\right)^{1/2}\left[E_n\{\psi_0(W)^2\}\right]^{1/2} \\
    &\leq \left(E_n[\{\hat{\psi}(W)-\psi_0(W)\}^2]\right)^{1/2}\left[|E_n\{\psi_0(W)^2\}-\sigma^2|+\sigma^2\right]^{1/2} \\
    &\leq \left(E_n[\{\hat{\psi}(W)-\psi_0(W)\}^2]\right)^{1/2} \left(\left[|E_n\{\psi_0(W)^2\}-\sigma^2|\right]^{1/2}+\sigma\right).
\end{align*}

\item High probability events.

From the previous step, we see that to control $|\hat{\sigma}^2-\sigma^2|$, it is sufficient to control two expressions:  $E_n[\{\hat{\psi}(W)-\psi_0(W)\}^2]$ and $|E_n\{\psi_0(W)\}^2-\sigma^2|$. These are controlled in Propositions~\ref{prop:Delta2} and~\ref{prop:other_half}, respectively. Therefore with probability $1-\epsilon'$,
$$
|\hat{\sigma}^2-\sigma^2|\leq \Delta'+2(\Delta')^{1/2}\{(\Delta'')^{1/2}+\sigma\}+\Delta''.
$$
    
\end{enumerate}
\end{proof}

\subsection{Corollary}

\begin{proof}[of Corollary~\ref{cor:CI}]
Immediately from $\Delta=o_p(1)$ in Theorem~\ref{thm:dml},
$$
\hat{\theta}=\theta_0+o_p(1),\quad
\sigma^{-1}n^{1/2}(\hat{\theta}-\theta_0)\leadsto \mathcal{N}(0,1),\quad 
 \text{pr}\left\{\theta_0 \text{ in }  \left(\hat{\theta}\pm \frac{\sigma}{n^{1/2}}\right)\right\}\rightarrow 1-a.$$
For the final result, it is sufficient that $\hat{\sigma}^2=\sigma^2+o_p(1)$, which follows from $\Delta'=o_p(1)$ and $\Delta''=o_p(1)$ in Theorem~\ref{thm:var}.
\end{proof}

%% file: E_proofs_local.tex
\section{Proofs of lemmas}

\subsection{Riesz representers}

\begin{proof}[of Lemma~\ref{prop:RR}]
As the operator norm of $\gamma\mapsto E\{m(W,\gamma)\}$,
$$
\bar{M}=\inf [c\geq 0: |E\{m(W,\gamma)\}|\leq c  \text{ for all } \gamma \text{ in }\Gamma \text{ such that } \|\gamma\|_{\text{\normalfont pr,2}}=1].
$$
By Jensen's inequality and Assumption~\ref{assumption:cont},
$$
|E\{m(W,\gamma)\}| \leq [E \{m(W,\gamma)^2\}]^{1/2} \leq \left(\bar{Q} [E \{\gamma(W)^2\} ]^q\right)^{1/2} = \bar{Q}^{1/2} \|\gamma\|^q_{\text{\normalfont pr,2}}.
$$
Taking the supremum of both sides over $\gamma$ in $\Gamma$ such that $\|\gamma\|_{\text{\normalfont pr,2}}=1$, we conclude that $\bar{M}\leq \bar{Q}^{1/2}<\infty$. The rest of the claim is shown in \cite[Lemma 2.1]{chernozhukov2018global}.
\end{proof}

\begin{proof}[of Lemma~\ref{lemma:RR_exists}]
For Examples~\ref{ex:CATE} and~\ref{ex:RDD}, the result is immediate from standard propensity score and regression arguments. For Example~\ref{ex:elasticity}. the result follows from \cite[Proposition 3 and Example 5]{ichimura2021influence}. For Example~\ref{ex:deriv}, the result follows from integration by parts.
\end{proof}

\subsection{Global functionals}

\begin{proof}[of Lemma~\ref{lemma:global}]
We extend \cite[Lemma 3.3]{chernozhukov2018global}. Write
$$
\psi_0(W)=U_1+\alpha_0^{\min}(W)U_2.
$$
By law of iterated expectations,
\begin{align*}
    \sigma^2
&=E(U_1^2)+2E\{U_1\alpha_0^{\min}(W)U_2\}+ E(\{\alpha_0^{\min}(W)U_2\}^2) \\
&=E\{E(U_1^2 \mid W)\}+2E\{U_1\alpha_0^{\min}(W) E(U_2 \mid W)\}+ E\{\alpha_0^{\min}(W)^2 E(U_2^2 \mid W)\}.
\end{align*}
$E(U_2 \mid W)=0$ by definition of $U_2$. Note that
$$
0 \leq E\{E(U_1^2 \mid W)\} \leq \bar{c}^2,
$$
and
$$
\tilde{c}^2 \bar{M}^2 \leq E\{\alpha_0^{\min}(W)^2 E(U_2^2 \mid W)\} \leq \bar{c}^2\bar{M}^2.
$$
In summary
$$
\tilde{c}^2 \bar{M}^2 \leq \sigma^2 \leq \bar{c}^2 (1+\bar{M}^2).
$$
By triangle inequality,
\begin{align*}
   \|\psi_0\|_{\text{ \normalfont pr},q} &
   \leq \|U_1\|_{\text{ \normalfont pr},q} + \|\alpha_0^{\min}(W)U_2\|_{\text{ \normalfont pr},q}  \\
    &=\|U_1\|_{\text{ \normalfont pr},q} + [E\{\alpha_0^{\min}(W)^q E( U_2^q \mid W)\}]^{1/q} \\
    &\leq \bar{c}+ \bar{c} \|\alpha^{\min}_0\|_{\text{ \normalfont pr},q}   \\
    &\leq \bar{c}(1+c(\bar{M}^2\vee 1)).
\end{align*}
\end{proof}

\begin{proof}[of Lemma~\ref{lemma:rp}]
To begin, observe that
\begin{align*}
    \partial_d \{ f(d\mid x) \partial_d \gamma(d,x)\}
    &= \{\partial_df(d\mid x)\}\{\partial_d \gamma(d,x)\}+f(d\mid x) \{\partial^2_d\gamma(d,x)\}\\
    &= [\{\partial_d \log f(d \mid x)\}\{\partial_d \gamma(d,x)\}+ \partial^2_d\gamma(d,x)]f(d \mid x) \\
    &=-k_{\gamma}(d,x)f(d \mid x).
\end{align*}
Using integration by parts and the boundary condition together with this result,
\begin{align*}
    E [\{\partial_d \gamma(D,X)\}^2]
    &=\int \{\partial_d \gamma(d,x)\}^2 f(d \mid x) f(x) \mathrm{d}dx \\
    &=\int \{\partial_d \gamma(d,x)\} \{ f(d\mid x) \partial_d \gamma(d,x)\} f(x) \mathrm{d}dx \\
    &=-\int \gamma(d,x) \partial_d \{ f(d\mid x) \partial_d \gamma(d,x)\} f(x) \mathrm{d}dx \\
    &=\int \gamma(d,x) k_{\gamma}(d,x) f(d\mid x)f(x) \mathrm{d}dx \\
    &=E \{\gamma(D,X) k_{\gamma}(D,X)\} \\
    &\leq \|\gamma\|_{\text{ \normalfont pr},2} \|k_{\gamma}\|_{\text{ \normalfont pr},2},
\end{align*}
where the inequality is Cauchy Schwarz. The final results immediately follow from the definition of $k_{\gamma}$ and triangle inequality.
\end{proof}

\begin{proof}[of Lemma~\ref{lemma:cont}]
For Example~\ref{ex:CATE}, write
\begin{align*}
    E [\{\ell_h(V)\gamma(1,V,X)-\ell_h(V)\gamma(0,V,X)\}^2 ]\leq 2E\{\ell_h(V)^2\gamma(1,V,X)^2\}+2E\{\ell_h(V)^2\gamma(0,V,X)^2\}.
\end{align*}
Invoking the bounded weighting and propensity score assumptions,
\begin{align*}
   E\{\ell_h(V)^2\gamma(1,V,X)^2\}&=E \left\{\frac{D}{\pi_0(V,X)}\ell_h(V)^2\gamma(D,V,X)^2\right\} \leq C E \left\{\gamma(D,V,X)^2\right\}; \\
    E\{\ell_h(V)^2\gamma(0,V,X)^2\}&=E \left\{\frac{1-D}{1-\pi_0(V,X)}\ell_h(V)^2\gamma(D,V,X)^2\right\} \leq C E \left\{\gamma(D,V,X)^2\right\}.
\end{align*}

For Example~\ref{ex:RDD}, write
$$
E[\{\ell^{+}_{h}(D)\gamma_0(D,X)-\ell^{-}_{h}(D)\gamma_0(D,X)\}^2]\leq 2 E[\{\ell^{+}_{h}(D)\gamma_0(D,X)\}^2]+2E[\{\ell^{-}_{h}(D)\gamma_0(D,X)\}^2].
$$
Invoking the bounded weighting assumption,
\begin{align*}
    E[\{\ell^{+}_{h}(D)\gamma_0(D,X)\}^2]&\leq CE\{\gamma_0(D,X)\}^2; \\
     E[\{\ell^{-}_{h}(D)\gamma_0(D,X)\}^2]&\leq CE\{\gamma_0(D,X)\}^2. 
\end{align*}

For Examples~\ref{ex:elasticity} and~\ref{ex:deriv}, appeal to Lemma~\ref{lemma:rp}.
\end{proof}

\begin{proof}[of Lemma~\ref{lemma:bounded_RR_global}]
The result is immediate from Lemma~\ref{lemma:RR_exists}.
\end{proof}

\subsection{Local functionals}

\begin{proof}[of Lemma~\ref{lemma:local}]
We extend \cite[Lemma 3.4]{chernozhukov2018global}. We proceed in steps.
\begin{enumerate}
    \item Moment bounds. 
    
    As in the of Lemma~\ref{lemma:global},
$$
\sigma^2=E\{E(U_1^2 \mid W)\}+ E\{\alpha_0^{\min,h}(W)^2 E(U_2^2 \mid W)\}.
$$
Note that
$$
0 \leq E\{E(U_1^2 \mid W)\} \leq \bar{c}^2 \|\ell\|^2_{\text{pr},2},
$$
and
$$
\tilde{c}^2 \|\alpha_0^{\min,h}\|^2_{\text{pr,2}} \leq E\{\alpha_0^{\min,h}(W)^2 E(U_2^2 \mid W)\} \leq \bar{c}^2\|\alpha_0^{\min,h}\|^2_{\text{pr,2}}.
$$
In summary
$$
\tilde{c}^2 \|\alpha_0^{\min,h}\|^2_{\text{pr,2}}  \leq \sigma^2 \leq \bar{c}^2 (\|\ell\|^2_{\text{pr},2}+\|\alpha_0^{\min,h}\|^2_{\text{pr,2}}).
$$
As in the proof of Lemma~\ref{lemma:global},
\begin{align*}
  \|\psi_0\|_{\text{pr},q} &\leq   \|U_1\|_{\text{ \normalfont pr},q} + [E\{\alpha_0^{\min,h}(W)^q E( U_2^q \mid W)\}]^{1/q}  
  \leq  \bar{c} (\|\ell \|_{\text{pr,q}}+\|\alpha_0^{\min,h}\|_{\text{pr,q}}).
\end{align*}
Next we characterize $\|\alpha_0^{\min,h}\|_{\text{pr,q}}$ in terms of $\|\ell \|_{\text{pr,q}}$. Since $\alpha_0^{\min,h}(w)=\ell_h(w_j)\alpha_0^{\min}$,
$$
\tilde{\alpha} \|\ell\|_{\text{pr},q}\leq \|\alpha_0^{\min,h}\|_{\text{pr},q} \leq \check{\alpha} \|\ell\|_{\text{pr},q},\quad \|\alpha_0^{\min,h}\|_{\text{pr},2}=\bar{M}.
$$
In summary,
$$
\tilde{c}  \tilde{\alpha} \| \ell\|_{\text{pr},2}  \leq \sigma \leq \bar c \sqrt{1+ \check{\alpha}^2} \| \ell\|_{\text{pr},2}, \quad  \tilde{\alpha} \| \ell\|_{\text{pr},2}  \leq \bar{M} \leq \check{\alpha} \| \ell\|_{\text{pr},2},  \quad
\|\psi_0\|_{\text{pr},q}
\leq  \bar c (1 + \check{\alpha} ) \| \ell \|_{\text{pr},q}.
$$
    
    \item Taylor expansion.
    
    Consider the change of variables $u=(v'-v)/h$ so that $\mathsf{d} u  = h^{-1} \mathsf{d} v'$. Hence
    \begin{align*}
        \| \ell\|^q_{\text{pr},q} \omega^q &=
        \| \ell  \omega\|^q_{\text{pr},q} \\
        &=  \left\| h^{-1} K \left(\frac{v-v'}{h}\right) \right\|^q_{\text{pr},q} \\
        &= \int h^{-q}\left|K \left(\frac{v'-v}{h}\right)\right|^q f_V(v') \mathsf{d} v' \\
        &=  \int h^{-(q-1) }|K (u)|^q f_V(v - u h)  \mathsf{d} u .
    \end{align*}
It follows that
$$
 h^{-(q - 1)/q}  \tilde{f}^{1/q}  \left(\int |K|^q\right)^{1/q} 
 \leq \| \ell\|_{\text{pr},q} \omega 
 \leq  h^{-(q - 1)/q}  \bar f^{1/q}  \left(\tiny{\int} |K|^q\right)^{1/q}.
$$ 
Further, we have that
$$
\omega  = \int  h^{-1} K\left(\frac{v'-v}{h}\right) f_V(v') \mathsf{d} v' = \int   K(u) f_V(v- u h) \mathsf{d} u.
$$
Note that
$$
\int   K(u) f_V(v-0u) \mathsf{d} u=\int   K(u) f_V(v) \mathsf{d} u=f_V(v).
$$
Using the Taylor expansion in $h$ around $h=0$ and the Holder inequality, there exist some $\tilde{h}$ in $[0,h]$ such that
$$
|\omega - f_V(v)| =  \left|    h \int   K(u) \partial_v f_V(v- u \tilde h) u \mathsf{d} u \right |  \leq   h \bar f' \int |u|| K(u)| du.
$$
Hence there exists some $h_1$ in $(h,h_0)$ depending only on $(K, \bar f', \tilde{f}, \bar f)$ such that
$$ \tilde{f}/2  \leq \omega \leq 2 \bar f.$$
In summary,
$$
 h^{- (q - 1)/q}  \tilde{f}^{1/q}  \left(\int |K|^q\right)^{1/q} \frac{1}{2 \bar f} \leq \| \ell\|_{\text{pr},q}  \leq  h^{-(q - 1)/q}  \bar f^{1/q}  \left(\tiny{\int} |K|^q\right)^{1/q} \frac{2}{\tilde{f}}.
$$     
    
    \item Collecting results.

In summary, for all $h < h_1$
$$
\tilde{c}  \tilde{\alpha} \| \ell\|_{\text{pr},2}  \leq \sigma \leq \bar c \sqrt{1+ \check{\alpha}^2} \| \ell\|_{\text{pr},2}, \quad  \tilde{\alpha} \| \ell\|_{\text{pr},2}  \leq \bar{M} \leq \check{\alpha} \| \ell\|_{\text{pr},2},  \quad
\|\psi_0\|_{\text{pr},q}
\leq  \bar c (1 + \check{\alpha} ) \| \ell \|_{\text{pr},q},
$$
where
$$
 h^{- (q - 1)/q}  \tilde{f}^{1/q}  \left(\int |K|^q\right)^{1/q} \frac{1}{2 \bar f} \leq \| \ell\|_{\text{pr},q}  \leq  h^{-(q - 1)/q}  \bar f^{1/q}  \left(\tiny{\int} |K|^q\right)^{1/q} \frac{2}{\tilde{f}},
$$  
so
$$
\sigma \asymp \bar{M}\asymp \| \ell\|_{\text{pr},2},\quad \|\psi_0\|_{\text{pr,q}}\lesssim \|\ell\|_{\text{pr},q},\quad \|\ell\|_{\text{pr},q}\asymp h^{-(q-1)/q}.
$$
\end{enumerate}
\end{proof}

\begin{proof}[of Lemma~\ref{lemma:cont_local}]
We prove the result for Example~\ref{ex:CATE}. The result for Example~\ref{ex:RDD} is similar.

Without loss of generality, let $\bar{Q}_h$ be the smallest finite constant for which Assumption~\ref{assumption:cont} holds, i.e.
$$
\bar{Q}_h=\inf [c\geq 0: E\{m_h(W,\gamma)^2\}\leq c \|\gamma\|^2_{\text{\normalfont pr,2}} \text{ for all } \gamma \text{ in }\Gamma ].
$$
To begin, write
\begin{align*}
    E\{m_h(W,\gamma)^2\}&=E[\ell_h(V)^2\{\gamma(1,V,X)-\gamma(0,V,X)\}^2] \\
    &\leq 2 E\{\ell_h(V)^2\gamma(1,V,X)^2\}+2E\{\ell_h(V)^2\gamma(0,V,X)^2\}.
\end{align*}
Since $\pi_0(v,x)$ is bounded away from zero and one,
$$
E\{\ell_h(V)^2\gamma(1,V,X)^2\}=E\left\{\ell_h(V)^2 \frac{D}{\pi_0(V,X)}\gamma(D,V,X)^2\right\} \leq C E\left\{\ell_h(V)^2 \gamma(D,V,X)^2\right\}.
$$
Likewise for $E\{\ell_h(V)^2\gamma(0,V,X)^2\}$. In summary,
$$
E\{m_h(W,\gamma)^2\} \leq 2C E\left\{\ell_h(V)^2 \gamma(D,V,X)^2\right\}.
$$
Viewing the latter expression as an inner product in $\mathbb{L}_2$, it is maximized by alignment, i.e. taking 
$
\gamma(D,V,X)^2=\ell_h(V)^2.
$
Therefore
$$
\frac{E\{m_h(W,\gamma)^2\}}{E\{\gamma(W)^2\}} \leq \frac{2C E\left\{\ell_h(V)^4\right\}}{E\left\{\ell_h(V)^2\right\}}=2C \frac{\|\ell\|^4_{\text{pr,4}}}{\|\ell\|^2_{\text{pr,2}}}.
$$
Appealing to $\|\ell\|_{\text{pr},q}\asymp h^{-(q-1)/q}$ from the proof of Lemma~\ref{lemma:local},
$$
\frac{E\{m_h(W,\gamma)^2\}}{E\{\gamma(W)^2\}}\lesssim \frac{h^{-3}}{h^{-1}}=h^{-2}.
$$
\end{proof}

\begin{proof}[of Lemma~\ref{lemma:bounded_RR_local}]
Write
$$
\|\alpha_0^{\min,h}\|_{\infty}\leq \check{\alpha}\|\ell\|_{\infty}.
$$
By the proof of Lemma~\ref{lemma:local},
$$
\|\ell\|_{\infty}=\left\| \frac{1}{h\omega}K\left(\frac{v'-v}{h}\right) \right\|_{\infty}\leq \bar{K}\frac{1}{h\omega}\leq \bar{K}\frac{2}{h \tilde{f}}.
$$
Therefore
$$
\|\alpha_0^{\min,h}\|_{\infty}\leq \check{\alpha}\bar{K}\frac{2}{h \tilde{f}}\lesssim h^{-1}.
$$
\end{proof}

\begin{proof}[of Lemma~\ref{lemma:translate_RR}]
Write
\begin{align*}
    \mathcal{R}(\hat{\alpha}^h_{\ell})
    &=E[\{\hat{\alpha}^h_{\ell}(W)-\alpha^{\min,h}_0(W)\}^2\mid I^c_{\ell}] \\
    &=E[\{\ell_h(W_i)\hat{\alpha}_{\ell}(W)-\ell_h(W_i)\alpha^{\min}_0(W)\}^2\mid I^c_{\ell}] \\
    &\leq \|\ell_h\|^2_{\infty} E[\{\hat{\alpha}_{\ell}(W)-\alpha^{\min}_0(W)\}^2\mid I^c_{\ell}] \\
    &=\|\ell_h\|^2_{\infty} \mathcal{R}(\hat{\alpha}_{\ell}).
\end{align*}
Finally recall from the proof of Lemma~\ref{lemma:bounded_RR_local} that $\|\ell_h\|_{\infty}\lesssim h^{-1}$. An identical argument holds for $ \mathcal{P}(\hat{\alpha}^h_{\ell})$.
\end{proof}

\subsection{Approximation error}

\begin{proof}[of Lemma~\ref{lemma:approx}]
For completeness, we quote the proof of \cite[Lemma 3.6]{chernozhukov2018global}. Define the quantities
\begin{align*}
    \vartheta_1(h) &=  \int m(v') h^{-1} K\left(\frac{v-v'}{h} \right) f_V(v') \mathsf{d} v'= \int m(v - h u) K(u) f_V(v - hu)\mathsf{d} u;\\
    \vartheta_2(h) &=  \int  h^{-1} K\left(\frac{v-v'}{h} \right) f_V(v') \mathsf{d} v' = \int   K(u) f_V(v- u h) \mathsf{d} u.
\end{align*}
By $\int K  =1 $,
$$
\vartheta_1(0) = m(v) f_V(v), \quad \vartheta_2(0) = f_V(v). 
$$
Hence$$
\theta^h_0 = \frac{\vartheta_1(h) }{\vartheta_2(h)},  \quad \theta_0^{\lim} = \frac{\vartheta_1(0) }{\vartheta_2(0)} = m(v). $$
The standard argument to control the bias of the higher order kernels employs the Taylor expansion of order $\mathsf{v}$ in $h$ around $h=0$; see e.g. \cite[Lemma B2]{newey1994kernel}. Such an argument implies there exists some constant $A_{\mathsf{v}}$ that depends only on  $\mathsf{v}$ such that 
$$
| \vartheta_1(h) - \vartheta_1(0)| \leq A_{\mathsf{v}} h^{ \mathsf{v}} \bar g_\mathsf{v} \int |  u|^{\mathsf{v}} | K(u)| du,
$$$$
| \vartheta_2(h) - \vartheta_2(0)| \leq A_{\mathsf{v}} h^{\mathsf{v}} \bar f_\mathsf{v} \int |  u|^{\mathsf{v}} | K(u)| du.
$$
Then using the relation
\begin{align*}
    &\frac{\vartheta_1(h) }{\vartheta_2(h)} - \frac{\vartheta_1(0) }{\vartheta_2(0)} \\ 
    &=
 \vartheta^{-1}_2(0) \{\vartheta_1(h)  - \vartheta_1(0)\}+  \vartheta_1(0) \{\vartheta_2^{-1}(h)-\vartheta_2^{-1}(0)\}+   \{\vartheta_1(h)  - \vartheta_1(0)\}\{\vartheta_2^{-1}(h)-\vartheta_2^{-1}(0)\},
\end{align*}
we deduce that for all $h< h_1\leq h_0$,
$$
|\theta_0^h  - \theta_0^{\lim}| \leq  \left | \frac{\vartheta_1(h) }{\vartheta_2(h)} - \frac{\vartheta_1(0) }{\vartheta_2(0)}  \right | \leq C h^{\mathsf{v}}, 
$$
where $C$ and $h_1$ depend  only on $(K, \mathsf{v}, \bar g_{\mathsf{v}}$,  $\bar f_{\mathsf{v}}$, $\tilde{f})$.
\end{proof}

\begin{proof}[of Corollary~\ref{cor:CI_local}]
By Lemma~\ref{lemma:local}, write the regularity condition on moments as
$$
\left\{\left(\kappa/\sigma\right)^3+\zeta^2\right\}n^{-1/2}\lesssim \left\{\left(h^{-1/6}\right)^3+(h^{-3/4})^2\right\}n^{-1/2}\lesssim h^{-3/2} n^{-1/2}.
$$
By Lemmas~\ref{lemma:local},~\ref{lemma:cont_local}, and~\ref{lemma:bounded_RR_local}, write the first learning rate condition as
$$
\left(\bar{Q}^{1/2}+\bar{\alpha}/\sigma+\bar{\alpha}'\right)\{\mathcal{R}(\hat{\gamma}_{\ell})\}^{1/2} 
\lesssim \left(h^{-1}+h^{-1}/h^{-1/2}+\bar{\alpha}'\right)\{\mathcal{R}(\hat{\gamma}_{\ell})\}^{1/2} 
\lesssim \left(h^{-1}+\bar{\alpha}'\right)\{\mathcal{R}(\hat{\gamma}_{\ell})\}^{1/2}.
$$
By Lemma~\ref{lemma:translate_RR}, write the second learning rate condition as
$$
\bar{\sigma}\{\mathcal{R}(\hat{\alpha}^h_{\ell})\}^{1/2} \lesssim \bar{\sigma}h^{-1}\{\mathcal{R}(\hat{\alpha}_{\ell})\}^{1/2}.
$$
By Lemmas~\ref{lemma:local} and~\ref{lemma:translate_RR}, write the initial term in the third learning rate condition as
$$
\{n \mathcal{R}(\hat{\gamma}_{\ell}) \mathcal{R}(\hat{\alpha}^h_{\ell})\}^{1/2}  /\sigma 
\lesssim \{n \mathcal{R}(\hat{\gamma}_{\ell}) \mathcal{R}(\hat{\alpha}_{\ell})\}^{1/2}  h^{-1}/h^{-1/2}
=h^{-1/2}\{n \mathcal{R}(\hat{\gamma}_{\ell}) \mathcal{R}(\hat{\alpha}_{\ell})\}^{1/2}.
$$
Likewise for the other terms. The approximation error condition is immediate from Lemma~\ref{lemma:approx}.
\end{proof}